%% file: main.tex
\pdfoutput=1

\documentclass[english,11pt]{article}
\include{macros}

\begin{document}

\title{\Large Model Assessment and Selection under Temporal Distribution Shift}\blfootnote{Author names are sorted alphabetically.}

\author{Elise Han\thanks{Department of Computer Science, Columbia University. Email: \texttt{lh3117@columbia.edu}.}
	\and Chengpiao Huang\thanks{Department of IEOR, Columbia University. Email: \texttt{chengpiao.huang@columbia.edu}.}
	\and Kaizheng Wang\thanks{Department of IEOR and Data Science Institute, Columbia University. Email: \texttt{kaizheng.wang@columbia.edu}.}
}

\date{This version: June 2024}

\maketitle

\begin{abstract}
We investigate model assessment and selection in a changing environment, by synthesizing datasets from both the current time period and historical epochs. To tackle unknown and potentially arbitrary temporal distribution shift, we develop an adaptive rolling window approach to estimate the generalization error of a given model. This strategy also facilitates the comparison between any two candidate models by estimating the difference of their generalization errors. We further integrate pairwise comparisons into a single-elimination tournament, achieving near-optimal model selection from a collection of candidates. Theoretical analyses and numerical experiments demonstrate the adaptivity of our proposed methods to the non-stationarity in data.
\end{abstract}

\input{main_intro}

\input{main_setup}

\input{main_assess}

\input{main_select}

\input{main_experiments}

\input{main_discussions}

\section*{Acknowledgement}
Elise Han, Chengpiao Huang, and Kaizheng Wang's research is supported by an NSF grant DMS-2210907 and a startup grant at Columbia University.

\newpage 
\appendix

\input{appendix_proof_assess}

\input{appendix_proof_select}

\input{appendix_experiments}

{
\bibliographystyle{ims}
\bibliography{bib}
}

\end{document}

%% file: main_intro.tex
\section{Introduction}\label{sec-intro}

Traditionally, statistical learning theory assumes that models are trained and tested under the same data distribution. However, when practitioners train a model and deploy it into real environment, the environment often changes over time. Such temporal distribution shift may lead to serious decline in the model's quality. It is important to assess the model's performance in real time and detect possible degradation.

Moreover, one often needs to choose among multiple candidate models that result from different
learning algorithms (e.g., linear regression, random forests, neural networks) and hyperparameters
(e.g., penalty parameter, step size, time window for training). Temporal distribution shift poses a major challenge to model selection, as past performance may not reliably predict future outcomes. Learners usually have to work with limited data from the current time period and abundant historical data, whose distributions may vary significantly.

\paragraph{Main contributions.} In this paper, we develop principled approaches to model assessment and model selection under temporal distribution shift.
\begin{itemize}
\item (Model assessment) We propose a rolling window strategy that adaptively selects historical data to estimate a model's generalization error.
\item (Model selection) We then use the method above to compare any pair of models by applying it to the difference between their generalization errors. Based on this, we develop a single-elimination tournament procedure for selecting the best model from a pool of candidates.
\end{itemize}
Furthermore, we provide theoretical analyses and numericals experiments to show that our algorithms adapt to the unknown temporal distribution shift.

\paragraph{Related works.} Model assessment and selection are classical problems in statistical learning \citep{HTF09}. Hold-out and cross-validation are arguably the most popular methods in practice. However, in the presence of distribution shift, the validation data may no longer accurately represent the test cases. This challenge has attracted considerable attention over the past two decades \citep{QSS22}. Existing works mostly focused on the static scenario where the validation data consists of independent samples from a fixed distribution. These methods do not apply when the environment is continuously changing over time. Rolling windows offer a practical solution to this issue and have been widely adopted for learning under temporal changes \citep{BGa07, HKY15, MMM12, MUp23,MEU23,HWa23}. Our method automatically selects a window tailored to the underlying non-stationarity near a given time. The strategy is inspired by the \emph{Goldenshluger-Lepski method} for bandwidth selection in non-parametric estimation \citep{GLe08}.

Adaptation to a non-stationary environment has recently gained popularity in online learning, where the goal is to attain a small cumulative error over an extended time horizon \citep{HSe09, BGZ15, DGS15, JRS15, WLu21,GCa21,BZZ22}. A classical problem in this area is prediction from expert advice \citep{LWa94}, which aims to track the best expert in the long run. Recently, there have also been works on online model selection for bandits and reinforcement learning \citep{ALN17,LPM21}, with a similar goal of tracking the best base algorithm in the long run. 
In contrast, our approach focuses on evaluating a model or selecting from multiple candidates at a specific point in time, leveraging \emph{offline} data collected from the past. Its performance is measured over the current data distribution. Consequently, our problem can be viewed as one of transfer learning, with source data from historical epochs and target data from the current distribution. We point out that our algorithms for the offline setting can be used as a sub-routine for online model selection, where the candidate models may vary across different time points.

Our problem is also related to but different from change-point detection \citep{NHZ16, TOV20}. The latter typically assumes that distribution shifts occur only at a small number of times called \emph{change points}. Our setting allows changes to happen in every time period with arbitrary magnitudes, covering a much broader range of shift patterns.


\paragraph{Outline.} The rest of the paper is organized as follows. \Cref{sec-setup} describes the problem setup. \Cref{sec-assess} and \Cref{sec-select} present our algorithms for model assessment and model selection, respectively. \Cref{sec-experiments} conducts numerical experiments for the proposed algorithms on synthetic and real datasets. Finally, \Cref{sec-discussions} concludes the paper and discusses future directions.

\paragraph{Notation.} Let $\ZZ_+=\{1,2,...\}$ be the set of positive integers. For $n\in\ZZ_+$, let $[n]=\{1,2,...,n\}$. For $x\in\RR$, define $x_+ = \max\{x,0\}$. For non-negative sequences $\{a_n\}_{n=1}^{\infty}$ and $\{b_n\}_{n=1}^{\infty}$, we write $a_n=\cO(b_n)$ if there exists $C>0$ such that for all $n\in\ZZ_+$, $a_n \leq C b_n$. Unless otherwise stated, $a_n\lesssim b_n$ also represents $a_n=\cO(b_n)$. We write $a_n \asymp b_n$ if $a_n = \cO(b_n)$ and $b_n = \cO(a_n)$. We write $a_n = o ( b_n )$ if $a_n / b_n \to 0$ as $n \to \infty$.

%% file: main_setup.tex
\section{Problem Setup}\label{sec-setup}

Let $\cZ$ and $\functionclass$ be spaces of samples and models, respectively. Denote by $\loss :~ \functionclass \times \cZ \to \RR$ a loss function. The quantity $\loss ( f, z )$ measures the loss incurred by a model $f \in \functionclass$ on a sample $z \in \cZ$. 
At time $t \in \ZZ_+$, the quality of a model $f$ is measured by its \emph{generalization error} (also called \emph{risk} or \emph{population loss}) $L_t (f) = \EE_{z \sim \distP_t} \loss ( f , z )$ over the current data distribution $\distP_t$. When the environment changes over time, the distributions $\{ \distP_t \}_{t=1}^{\infty}$ can be different.

Suppose that at each time $t$, we receive a batch of $B_t\in\ZZ_+$ i.i.d.~samples $\datasetB_t = \{ z_{t, i} \}_{i=1}^{B_t}$ from $\distP_t$, independently of the history. We seek to solve the following two questions based on the data $\{ \datasetB_j \}_{j=1}^t$:

\begin{problem}[Model assessment]\label{problem-assessment}
Given a model $f \in \functionclass$, how to estimate its population loss $L_t (f)$?
\end{problem}

\begin{problem}[Model selection]\label{problem-selection}
Given a collection of candidate models $\{ f_r \}_{r=1}^m \subseteq \functionclass$, how to select one with a small population loss? In other words, we want to choose $\widehat{r}  \in [m]$ so that $L_t (f_{\widehat{r}  }) \approx \min_{ r \in [m]} L_t (f_r )$.
\end{problem}

%% file: main_assess.tex
\section{Model Assessment}\label{sec-assess}

In this section we study Problem \ref{problem-assessment}, under the assumption that the loss $\loss$ is bounded.

\begin{assumption}[Bounded loss]\label{assumption-loss-bounded}
The loss function $\loss$ takes values in a given interval $[a, b]$.
\end{assumption}

In fact, we will consider a more general problem.

\begin{problem}[Mean estimation]\label{problem-mean}
Let $\{ \distQ_j \}_{j=1}^{t}$ be probability distributions over $[a, b]$ and $\{ \dataset_j \}_{j=1}^t$ be independent datasets, where $\dataset_j = \{ u_{j, i} \}_{i=1}^{B_j}$ consists of $B_j \geq 1$ i.i.d.~samples from $\distQ_j$. Given $\{ \dataset_j \}_{j=1}^t$, how to estimate the expectation $\mean_t$ of $\distQ_t$?
\end{problem}

Problem \ref{problem-assessment} is a special case of Problem \ref{problem-mean} with $u_{j, i} = \loss ( f , z_{j, i} )$ and $\mean_t=L_t(f)$. To tackle Problem \ref{problem-mean}, a natural idea is to average data from the current period. More generally, we consider a \emph{look-back window} of the $k$ most recent periods, and approximate $\mean_t$ by their sample average:
\begin{align}
\meanhat_{t, k} = \left(\sum_{j=t-k+1}^t B_j\right)^{-1} \sum_{ j=t-k+1 }^t \sum_{i=1}^{B_j} u_{j, i}.
\label{eqn-mu-k-0}
\end{align}

To measure the quality of $\meanhat_{t, k}$, we invoke the Bernstein inequality \citep{BLM13}. See \Cref{sec-lem-Hoeffding-Bernstein-proof} for its proof.

\begin{lemma}[Bernstein bound]\label{lem-Hoeffding-Bernstein}
	Let $\{ x_i \}_{i=1}^n $ be independent random variables taking values in $[a, b]$ almost surely. Define the average variance $\sd^2 = \frac{1}{n} \sum_{i=1}^n \var (x_i)$. For any $\delta \in (0 , 1 )$, with probability at least $1-\delta$, 
	\begin{align*}
	& \bigg| \frac{1}{n} \sum_{i=1}^{n} ( x_i - \EE x_i ) \bigg| 
	\le \sd \sqrt{ \frac{ 2 \log ( 2 / \delta)  }{ n } } + \frac{ 2  (b-a) \log (2 / \delta) }{3 n }.
	\end{align*}
\end{lemma}

Denote by $\mean_j$ and $\sd_j^2$ the mean and the variance of $\distQ_j$, respectively. Let $B_{t, k} = \sum_{ j=t-k+1 }^t B_j$,
\[
\mean_{t,k} = \frac{1}{B_{t,k}}\sum_{j=t-k+1}^tB_j\mean_j
\quad\text{and}\quad
\sd_{t,k}^2 = \frac{1}{B_{t,k}}\sum_{j=t-k+1}^tB_j\sd_j^2.
\]
Using \Cref{lem-Hoeffding-Bernstein} and the triangle inequality 
\[
|\meanhat_{t, k} - \mean_{t}| 
\leq 
|\meanhat_{t, k} - \mean_{t, k}| 
+ | \mean_{t, k} - \mean_t |,
\] 
we obtain the following bias-variance decomposition of the approximation error of $\widehat{\mean}_{t,k}$.

\begin{corollary}\label{cor-Hoeffding-Bernstein}
For $t \in \ZZ_+$, $k \in [t]$ and $ \delta \in (0, 1) $, define $M = b  - a$,
\begin{align*}
& \phi (t , k ) = \max_{t-k+1\le j\le t} | \mean_j - \mean_{t} | , \\[6pt]
& \psi(t,k,\delta) = 
\begin{cases}
M ,&\quad\text{if } B_{t,k}=1 \\[6pt]
\displaystyle\sd_{t, k} \sqrt{ \frac{2 \log( 2 / \delta) }{ B_{t,k} } } + \frac{ 2 M \log ( 2 / \delta) }{ 3 B_{t, k} }
,&\quad \text{if } B_{t,k}\ge 2
\end{cases}.
\end{align*}
With probability at least $1-\delta$,
\[
|\meanhat_{t , k} - \mean_t| \leq \phi (t , k ) + \psi(t, k, \delta).
\]
\end{corollary}

Here $\phi(t,k)$ upper bounds the bias induced by using data from the $k$ most recent periods, while $\psi(t,k,\delta)$ upper bounds the statistical uncertainty associated with $\meanhat_{t,k}$. Ideally, we would like to choose a window $k^*$ that minimizes the bias-variance decomposition $\phi (t , k ) + \psi(t, k, \delta)$. This optimal window $k^*$ depends on the pattern of distribution shift: it is large when the environment is near-stationary, and is small when the environment is heavily fluctuating. However, such structural information is generally unavailable in practice. This is reflected in the fact that both $\phi$ and $\psi$ involve unknown quantities $|\mean_j-\mean_k|$ and $\sigma_{t,k}$ that depend on the unknown distribution shift. As a consequence, direct minimization of $\phi (t , k ) + \psi(t, k, \delta)$ over $k\in[t]$ is infeasible. In the following, we will construct proxies for $\psi$ and $\phi$ that are computable from data.

We first construct a proxy for $\psi(t, k, \delta)$. A natural approximation of $\sd_{t, k}^2$ is the sample variance over the $k$ most recent periods, given by
\begin{equation}
\totalsdhat_{t, k}^2 = \frac{1}{B_{t, k} - 1} \sum_{j=t-k+1}^t \sum_{ i = 1}^{B_j} ( u_{j, i} - \meanhat_{t, k}  )^2 . \label{eqn-sigma-k-0} 
\end{equation}
It has been used for deriving empirical versions of the Bernstein inequality \citep{AMS07,MPo09}. Based on the estimate $\totalsdhat_{t, k}^2$, we define our proxy as
\begin{equation}
\widehat{\psi}(t, k, \delta)
= 
\begin{cases}
M ,&\quad\text{if } B_{t,k}=1 \\[6pt]
\displaystyle \totalsdhat_{t, k}  \sqrt{ \frac{2 \log( 2 / \delta) }{ B_{t, k} } } + \frac{ 8 M \log ( 2 / \delta) }{ 3 ( B_{t, k} - 1 ) }
,&\quad \text{if } B_{t,k}\ge 2
\end{cases}.
\label{eqn-psi-hat}
\end{equation}

As \Cref{lem-Bernstein-variance} shows that with high probability, $\widehat{\psi}$ upper bounds $\psi$ and their gap is not too large. Its proof is deferred to \Cref{sec-lem-Bernstein-variance-proof}.

\begin{lemma}\label{lem-Bernstein-variance}	
For any $\delta\in(0,1)$, define $ $
	\[
	\xi (t, k, \delta) =
	\begin{cases}
	0,&\quad\text{if } B_{t,k} = 1 \\[6pt]
	\displaystyle\sqrt{ \frac{ 4 \log ( 2 / \delta ) }{B_{t, k}  } }  
	\max_{t-k+1\le j\le t} | \mean_{j} - \mean_{t} |
	+
	\frac{ 13 (b-a) \log(2 / \delta) }{ 3( B_{t, k} - 1 ) },&\quad\text{if } B_{t,k}\ge 2
	\end{cases}.
	\]
	Then
	\begin{align*}
	& \PP \left(
	\psi (t , k , \delta ) \leq
	\widehat\psi (t , k , \delta ) 
	\right) \ge 1 - \delta, \\[4pt]
	& \PP \left(
	\widehat\psi (t , k , \delta ) 
	\leq  \psi (t , k , \delta ) + \xi (t , k , \delta) 
	\right) \ge 1 - \delta .
	\end{align*}
\end{lemma}
	
Combining \Cref{cor-Hoeffding-Bernstein} with the first bound in \Cref{lem-Bernstein-variance} immediately gives the following useful corollary.

\begin{corollary}\label{cor-assess-bias-variance-empirical}
Let $\delta\in(0,1)$. With probability at least $1-2t\delta$, it holds
\begin{equation}
|\meanhat_{t , k} - \mean_t| \leq \phi (t , k ) + \widehat\psi(t, k, \delta) ,\quad \forall k \in [t].
\label{eqn-assess-bias-variance-empirical}
\end{equation}
\end{corollary}

To construct a proxy for $\phi(t, k)$, we borrow ideas from the Goldenshluger-Lepski method for adaptive non-parametric estimation \cite{GLe08}. Define
\begin{equation}
\widehat\phi (t, k, \delta)= \\
\max_{i \in [k]} \bigg( | \meanhat_{t, k} - \meanhat_{t, i } |
- \left[ \widehat{\psi}\left(t, k , \delta \right) + \widehat{\psi}\left(t, i , \delta \right) \right] \bigg)_+.
\label{eqn-phi-hat}
\end{equation}
We now give an interpretation of $\widehat\phi$. In light of the bias-variance decomposition in \Cref{cor-Hoeffding-Bernstein}, the quantity 
\begin{equation}
| \meanhat_{t, k} - \meanhat_{t, i } | - [ \widehat{\psi}\left(t, k , \delta \right) + \widehat{\psi}\left(t, i , \delta \right) ]
\label{eqn-phi-hat-heart}
\end{equation}
can be viewed as a measure of the bias between the windows $k$ and $i$, where subtracting $\widehat{\psi}\left(t, k , \delta \right)$ and $\widehat{\psi}\left(t, i , \delta \right)$ eliminates the stochastic error and teases out the bias. Indeed, as the following lemma shows, $\widehat{\phi}(t,k,\delta)\le 2\phi(t,k)$ holds with high probability. Its proof is given in \Cref{sec-lem-bias-proxy-vs-bias-proof}.
\begin{lemma}\label{lem-bias-proxy-vs-bias}
When the event \eqref{eqn-assess-bias-variance-empirical} happens,
\[
0\le\widehat{\phi}(t,k,\delta)\le 2\phi(t,k).
\]
\end{lemma}

We take the positive part in \eqref{eqn-phi-hat} so that when the quantity \eqref{eqn-phi-hat-heart} is negative, we regard the bias as dominated by the stochastic error and hence negligible. Taking maximum over all windows $i\in[k]$ makes sure that we detect all possible biases between window $k$ and the smaller windows. 

Replacing $\psi$ and $\phi$ with their proxies $\widehat{\psi}$ and $\widehat\phi$ gives \Cref{alg-mean}. We note that the quantities $\meanhat_{t, k}$ and $\totalsdhat_{t, k}$ can be conveniently computed from summary statistics of individual datasets. Define empirical first and second momends in the $j$-th time period, $\meanhat_j =  B_j^{-1} \sum_{i=1}^{B_j} u_{j, i} $ and $\widehat\omega_j = B_j^{-1} \sum_{i=1}^{B_j} u_{j, i}^2$. Then
	\begin{align}
	& \meanhat_{t, k} = \frac{1}{B_{t, k}} \sum_{ j=t-k+1 }^t B_{j} \meanhat_{j}  , 
	\label{eqn-mu-k} \\
	& \totalsdhat_{t, k}^2 
	= \frac{B_{t, k}}{B_{t, k} - 1} \cdot 
	\bigg(
	\frac{1}{B_{t, k}} \sum_{ j=t-k+1 }^t B_{j} \widehat\omega_{j}  -  \meanhat_{t, k}^2
	\bigg) .
	\label{eqn-sigma-k} 
	\end{align}

\begin{algorithm}[t]
	\begin{algorithmic}
	\STATE {\bf Input:} Data $\{ \dataset_j \}_{j=1}^t$, hyperparameters $\delta'$ and $M$. 
	\\
	\FOR{$k = 1,\cdots, t$}
	\STATE Compute $\meanhat_{t,k}$, $\totalsdhat_{t,k}$, $\widehat{\psi}(t,k,\delta')$ and $\widehat\phi(t,k,\delta')$ according to \eqref{eqn-mu-k-0}, \eqref{eqn-sigma-k-0}, \eqref{eqn-psi-hat} and \eqref{eqn-phi-hat}.
	\ENDFOR
	\STATE Choose any $	\widehat{k} \in \argmin_{ k \in [t] }  \{ \widehat\phi (t, k, \delta')  + \widehat{\psi} (t , k , \delta')  \}$.
	\STATE {\bf Output:} $\meanhat_{ t , \widehat{k} }$.
	\caption{Adaptive Rolling Window for Mean Estimation (Problem \ref{problem-mean})}
	\label{alg-mean}
	\end{algorithmic}
\end{algorithm}

We now present theoretical guarantees for \Cref{alg-mean}.
\begin{theorem}[Oracle inequality]\label{thm-assess-oracle}
Let Assumption \ref{assumption-loss-bounded} hold. Choose $\delta \in (0, 1)$ and take $\delta' = \delta/(3t)$ in \Cref{alg-mean}. With probability at least $1 - \delta $, the output of \Cref{alg-mean} satisfies
\begin{align*}
| \meanhat_{t, \widehat{k} } - \mean_{t} | \lesssim
\min_{k \in [t]} \{
\phi (t , k ) + \psi(t, k, \delta) 
\}.
\end{align*}
Here $\lesssim$ only hides a logarithmic factor of $t$ and $\delta^{-1}$.
\end{theorem}

\Cref{thm-assess-oracle} states that the selected window $\widehat{k}$ is near-optimal for the error bound in \Cref{cor-Hoeffding-Bernstein} derived from bias-variance decomposition.  We illustrate this oracle property under the following distribution shift patterns.

\begin{example}[Change point]\label{eg-local-stationarity}
Suppose that the environment remained unchanged over the last $K$ periods but had been very different before, i.e.~$\distQ_{t-K}\neq \distQ_{t-K+1} = \cdots = \distQ_t$. If $K$ were known, one could take the window size $K$ and output $\meanhat_{t,K}$ as an estimate of $\mean_t$. By \Cref{thm-assess-oracle}, $\meanhat_{t, \widehat{k} } $ is at least comparable to $\meanhat_{t,K}$ in terms of the estimation error: up to a logarithmic factor,
\[
|\meanhat_{t,\widehat{k}} - \mu_t| \lesssim \frac{\sd_{t} }{ \sqrt{ B_{t,K } } } + \frac{ 1 }{ B_{t, K}  }.
\] 
Therefore, \Cref{alg-mean} automatically adapts to the local stationarity and is comparable to using $B_{t, K}$ i.i.d.~samples.
\end{example}

\begin{example}[Bounded drift]\label{eg-bounded-drift}
Suppose that the distribution shift between consecutive periods is bounded, i.e.\ there exists $\Delta>0$ such that for each $j\in\ZZ_+$, $|\mu_{j+1}-\mu_j|\le\Delta$. This is a common assumption in the literature \citep{Bar92, HLo94, BLo97, MMM12}. The quantity $\Delta$ characterizes the non-stationarity: as $\Delta$ grows larger, the environment is allowed to fluctuate more wildly. For simplicity, we further assume $B_j=1$ for all $j\in\ZZ_+$. In this case, $\phi(t,k) \le (k-1)\Delta$, so the bias-variance decomposition in \Cref{cor-Hoeffding-Bernstein} becomes
\[
|\widehat{\mean}_{t,k} - \mean| \lesssim (k-1)\Delta + \sqrt{\frac{1}{k}}
\]
up to a logarithmic factor. If $\Delta$ were known, then one could pick the optimal window size $k^* \asymp \Delta^{-2/3}$, with an estimation error of $\cO(\Delta^{1/3})$ which is known to be optimal \citep{BLo97}. \Cref{thm-assess-oracle} shows that without knowing $\Delta$, \Cref{alg-mean} achieves the optimal order of estimation error.
\end{example}

We now turn to the proof of \Cref{thm-assess-oracle}. A key ingredient is the following lemma, which can be seen as an empirical version of \Cref{thm-assess-oracle} with $\psi$ replaced by $\widehat{\psi}$.

\begin{lemma}\label{lem-GL}
When the event \eqref{eqn-assess-bias-variance-empirical} happens,
\[
 |\widehat{\mu}_{t,\widehat{k}} - \mu_t |
\le 
3\min_{k\in[t]}\left\{
\phi(t,k) + \widehat{\psi}(t,k,\delta)
\right\}.
\]
\end{lemma}

\begin{proof}
For any $k \in [ \widehat{k} ]$,
\begin{align*}
|\widehat{\mean}_{t,\widehat{k}} - \mean_t| 
&
\le 
|\widehat{\mean}_{t,\widehat{k}} - \widehat{\mean}_{t,k}| + |\widehat{\mean}_{t,k} - \mean_t| \\[4pt]
&\le 
\left[ \widehat{\phi}(t,\widehat{k},\delta) + \widehat{\psi}(t,\widehat{k},\delta) + \widehat{\psi}(t,k,\delta) \right] +
\left[\phi(t,k) + \widehat{\psi}(t,k,\delta)\right] \tag{by \eqref{eqn-phi-hat} and \eqref{eqn-assess-bias-variance-empirical} } \\[4pt]
&=
\left[\widehat{\phi}(t,\widehat{k},\delta) + \widehat{\psi}(t,\widehat{k},\delta)\right] + \left[\phi(t,k) + 2\widehat{\psi}(t,k,\delta)\right] \\[4pt]
&\le 
\left[\widehat{\phi}(t,k,\delta) + \widehat{\psi}(t,k,\delta)\right] + \left[\phi(t,k) + 2\widehat{\psi}(t,k,\delta)\right] \tag{by the definition of $\widehat{k}$} \\[4pt]
&\le 
\left[2\phi(t,k) + \widehat{\psi}(t,k,\delta)\right] + \left[\phi(t,k) + 2\widehat{\psi}(t,k,\delta)\right] \tag{by \Cref{lem-bias-proxy-vs-bias}} \\[4pt]
&=
3\left[\phi(t,k) + \widehat{\psi}(t,k,\delta)\right].
\end{align*}
On the other hand, for any $k \in \{ \widehat{k} + 1 , \cdots , t\}$,
\begin{align*}
|\widehat{\mean}_{t,\widehat{k}} - \mean_t| 
&\le 
\phi(t,\widehat{k}) + \widehat{\psi}(t,\widehat{k},\delta) \tag{by \eqref{eqn-assess-bias-variance-empirical}} \\[4pt]
&\le 
\phi(t,k) + \left[\widehat{\psi}(t,\widehat{k},\delta) + \widehat{\phi}(t,\widehat{k},\delta)\right] \tag{by \Cref{lem-bias-proxy-vs-bias}} \\[4pt]
&\le 
\phi(t,k) + \left[\widehat{\psi}(t,k,\delta) + \widehat{\phi}(t,k,\delta)\right] \tag{by the definition of $\widehat{k}$} \\[4pt]
&\le 
3\left[\phi(t,k) + \widehat{\psi}(t,k,\delta)\right]. \tag{by \Cref{lem-bias-proxy-vs-bias}}
\end{align*}
The proof is finished by taking the minimum over $k\in[t]$. 
\end{proof}

Combining \Cref{lem-GL} and the second bound in \Cref{lem-Bernstein-variance} proves \Cref{thm-assess-oracle}. We provide a full proof in \Cref{sec-thm-assess-oracle-proof}, along with a more precise bound.

\begin{remark}[Boundedness of loss]
The boundedness Assumption \ref{assumption-loss-bounded} can be relaxed to a light-tailed condition. For instance, we may assume that $\ell(f,z)$ is sub-exponential for $f\in\mathcal{F}$ and $z\sim \distP_j$, and then apply a standard truncation argument. Note that for any sub-exponential random variables $\{ v_i \}_{i=1}^n$, the bound $\max_{i\in[n]}|v_i| = \mathcal{O}(\log n)$ holds with high probability. Hence, we can truncate the random variables at a logarithmic level and then apply Bernstein's inequality. This will only incur an extra logarithmic factor in the oracle inequality.
\end{remark}

\begin{remark}[Logarithmic factor in the oracle inequality]
The $\cO(\log t)$ dependence in \Cref{thm-assess-oracle} comes from taking union bound on \Cref{cor-Hoeffding-Bernstein} over all $k\in[t]$. It is in fact possible to improve it to a sharp bound $\cO(\log\log t)$, by using advanced techniques from the law of iterated logarithm and martingale concentration \citep{JMN14,BRa16}.
\end{remark}


%% file: main_select.tex
\section{Model Selection}\label{sec-select}

In this section, we consider Problem \ref{problem-selection}. We will first study the case of selection between two models, and then extend the approach to the general case of $m\in\ZZ_+$ models.

\subsection{Warmup: Model Comparison}

We first consider the case $m=2$, where the goal is to compare two models $f_1$ and $f_2$, and choose the better one. 
As in \Cref{sec-assess}, using a look-back window $k\in[t]$, we can estimate $L_t(f)$ by
\[
\widehat{L}_{t,k}(f) = \frac{1}{B_{t,k}}\sum_{j=t-k+1}^t\sum_{i=1}^{B_j}\loss(f,z_{j,i}).
\]
We will choose $k\in[t]$ and return 
\[
\widehat{r}_k\in\argmin_{r\in[2]}\widehat{L}_{t,k}(f_r).
\] 
Our approach is based on the following key observation, proved in \Cref{sec-lem-from-assess-to-compare-proof}.

\begin{lemma}\label{lem-from-assess-to-compare}
For every $k\in[t]$, the index $\widehat{r}_k$ satisfies
\[
L_t(f_{\widehat{r}_k}) - \min_{r\in[2]}L_t(f_r) \\
\le 
\left|\big[\widehat{L}_{t,k}(f_1) - \widehat{L}_{t,k}(f_2)\big] - \big[L_t(f_1) - L_t(f_2)\big]\right|.
\]
\end{lemma}

Define
\[
\Delta_j = L_j(f_1) - L_j(f_2)
\] 
and $\widehat{\Delta}_{t,k}=\widehat{L}_{t,k}(f_1) - \widehat{L}_{t,k}(f_2)$. Then, $\Delta_j$ is the performance gap between $f_1$ and $f_2$ at time $j$, and $\widehat{\Delta}_{t,k}$ is a sample average. By \Cref{lem-from-assess-to-compare}, it suffices to choose $k$ such that $|\widehat{\Delta}_{t,k} - \Delta_t|$ is small. That is, \emph{an accurate estimate of the performance gap guarantees near-optimal selection}.

This reduces the problem to Problem \ref{problem-mean}, with $u_{j,i} = \loss(f_1,z_{j,i}) - \loss(f_2,z_{j,i})$ and thus $\mean_t = \Delta_t$. We can then readily apply \Cref{alg-mean}. The detailed description is given in \Cref{alg-compare}.

\begin{algorithm}[t]
	\begin{algorithmic}
\STATE {\bf Input:} Models $f_1$ and $f_2$, data $\{ \datasetB_j \}_{j=1}^t$, hyperparameters $\delta'$ and $M$. 
	\STATE Let $u_{j,i} = \loss(f_1,z_{j,i}) - \loss(f_2,z_{j,i})$ and $\dataset_j = \{u_{j,i}\}_{i=1}^{B_j}$.
	\STATE Run \Cref{alg-mean} with inputs $\{\dataset_j\}_{j=1}^t$, $\delta'$ and $M$ to obtain $\widehat{\Delta}_{t,\widehat{k}}$. 
	\STATE Let 
	\[
	\widehat{r}_{\widehat{k}} = 
	\begin{cases}
	1,&\quad\text{if } \widehat{\Delta}_{t,\widehat{k}} \le 0 \\
	2,&\quad\text{if } \widehat{\Delta}_{t,\widehat{k}} > 0
	\end{cases}.
	\]
	\STATE {\bf Output:} $\widehat{f} = f_{\widehat{r}_{\widehat{k}}}$.
	\caption{Adaptive Rolling Window for Model Comparison}
	\label{alg-compare}
	\end{algorithmic}
\end{algorithm}

\Cref{thm-assess-oracle} and \Cref{lem-from-assess-to-compare} directly yield the following guarantee of \Cref{alg-compare}.

\begin{theorem}[Oracle inequality]\label{thm-compare-oracle-slow}
Let Assumption \ref{assumption-loss-bounded} hold. Choose $\delta\in(0,1)$ and take $\delta' = \delta / (3t)$ in \Cref{alg-compare}. With probability at least $1-\delta$, \Cref{alg-compare} outputs $\widehat{f}$ satisfying
\[
L_t(\widehat{f}) - \min_{r\in[2]}L_t(f_r) \\
\lesssim
\min_{k\in[t]}\left\{\max_{t-k+1\le j\le t}|\Delta_j - \Delta_t|
+
\frac{\widetilde{\sigma}_{t,k}}{\sqrt{B_{t,k}}} + \frac{1}{B_{t,k}}
\right\},
\]
where 
\[
\widetilde{\sigma}_{t,k}^2 = \frac{1}{B_{t,k}}\sum_{j=t-k+1}^t B_j \var_{z\sim\distP_j} [ \ell(f_1,z) - \ell(f_2,z) ]
\]
and $\lesssim$ only hides a logarithmic factor of $t$ and $\delta^{-1}$.
\end{theorem}

Consider again \Cref{eg-local-stationarity} where $\distP_{t-K} \neq \distP_{t-K+1}=\cdots=\distP_t$ for some $K$. \Cref{thm-compare-oracle-slow} admits a similar interpretation as \Cref{thm-assess-oracle}: \Cref{alg-compare} selects $\widehat{f}$ satisfying
\begin{equation}
L_t(\widehat{f}) - \min_{r\in[2]}L_t(f_r)
\lesssim
\frac{\widetilde{\sigma}_t}{\sqrt{B_{t,K}}} + \frac{1}{B_{t,K}}
\lesssim \frac{1}{\sqrt{B_{t,K}}} 
,
\label{eqn-compare-slow-rate-eg}
\end{equation}
where $\widetilde{\sigma}_t^2 = \var_{z\sim\distP_t}\left[\ell(f_1,z) - \ell(f_2,z)\right]$. 

In the setting of bounded regression without covariate shift, we may further improve the rate in \eqref{eqn-compare-slow-rate-eg}. To state the results, we let $\cX$ be a feature space and consider the following assumptions.

\begin{assumption}[Bounded response]\label{assumption-bounded-regression}
For $j\in\ZZ_+$, a sample $z\sim\distP_j$ takes the form $z=(x,y)$, where $x \in \cX$ is the covariate and $y \in \RR$ is the response. There exists $M_0>0$ such that $|y|\le M_0$ holds for $(x,y)\sim\distP_j$ and $j\in\ZZ_+$.
\end{assumption}

\begin{assumption}[Bounded models]\label{assumption-bounded-functions}
The loss $\loss$ is given by $\loss(f,(x,y)) = [f(x)-y]^2$. For all $x \in \cX$ and $f\in\functionclass$, $|f(x)|\le M_0$.
\end{assumption}

\begin{assumption}[No covariate shift]\label{assumption-no-covariate-shift}
The distributions $\{ \distP_j \}_{j=1}^{\infty}$ have the same marginal distribution of the covariate, denoted by $\distP$.
\end{assumption}

Assumptions \ref{assumption-bounded-regression} and \ref{assumption-bounded-functions} imply Assumption \ref{assumption-loss-bounded} with $a=0$ and $b=4M_0^2$. In Assumption \ref{assumption-no-covariate-shift}, the distribution of the covariate $x_{j,1}$ is the same for all $j\in\ZZ_+$, while the conditional distribution of $y_{j,1}$ given $x_{j,1}$ may experience shifts. The latter is commonly known as \emph{concept drift} \citep{GZI14}. 

Before we state the result, we introduce a few notations. Define $f_j^*(\cdot) = \EE ( y_{j, 1} | x_{j, 1} = \cdot )$, which minimizes the mean square error $\EE [f(x)-y]^2$ over the class of all measurable $f:~\cX\to \RR$. For $h_1,h_2\in\functionclass$, define an inner product
\[
\langle h_1, h_2 \rangle_{\distP} = \EE_{x \sim \distP} \left[ h_1(x) h_2(x) \right],
\]
which induces a norm $\| h \|_{\distP} = \sqrt{ \langle h, h \rangle_{\distP}}$. It can be readily checked that $L_t(f)-L_t(f_t^*) = \|f-f_t^*\|_{\distP}^2$ for all $f\in\functionclass$. Thus, we may measure the performance of a model $f\in\functionclass$ by
\[
\|f-f_t^*\|_{\distP} = \sqrt{L_t(f) - L_t(f_t^*)},
\] 
as it admits the interpretation of being both the distance between $f$ and $f_t^*$ under $\|\cdot\|_{\distP}$, and the square root of the excess risk $L_t(f) - L_t(f_t^*)$. For $j\neq t$, the quantity $\|f_j^*-f_t^*\|_{\distP}$ serves as a measure of distribution shift between time $j$ and time $t$.

We are now ready to state our result. In \Cref{sec-thm-compare-oracle-proof} we provide a more precise bound and its proof. 

\begin{theorem}[Fast rate]\label{thm-compare-oracle}
Let Assumptions \ref{assumption-bounded-regression}, \ref{assumption-bounded-functions} and \ref{assumption-no-covariate-shift} hold. Let $M_0$ be a constant. Choose $\delta\in(0,1)$ and take $\delta' = \delta / (3t)$ in \Cref{alg-compare}. With probability at least $1-\delta$, \Cref{alg-compare} outputs $\widehat{f}$ satisfying
\[
\|\widehat{f}-f_t^*\|_{\distP} - \min_{r\in[2]}\|f_r-f_t^*\|_{\distP} \\
\lesssim 
\min_{k\in[t]}\left\{\max_{t-k+1\le j\le t}\|f_j^*-f_t^*\|_{\distP} + \frac{1}{ \sqrt{B_{t,k}} }\right\}.
\]
Here $\lesssim$ only hides a logarithmic factor of $t$ and $\delta^{-1}$.
\end{theorem}

The oracle inequality in \Cref{thm-compare-oracle} shares the same bias-variance structure as that of \Cref{thm-compare-oracle-slow}. By squaring both sides of the bound, we see that in \Cref{eg-local-stationarity} where $\cP_{t-K}\neq\cP_{t-K+1}=\cdots=\cP_t$, 
\begin{equation}
L_t(\widehat{f}) - L_t(f_t^*)
\lesssim
\min_{r\in[2]}L_t(f_r) - L_t(f_t^*)
+
\frac{1}{B_{t,K}}.
\label{eqn-compare-fast-rate-eg}
\end{equation}
When either $f_1$ or $f_2$ has a small error so that $\min_{r\in[2]}L_t(f_r) - L_t(f_t^*) = o( 1 / \sqrt{ B_{t, K} } )$, \eqref{eqn-compare-fast-rate-eg} provides a much sharper guarantee on $L_t(\widehat{f})$ compared to \eqref{eqn-compare-slow-rate-eg}. As the proof of \Cref{thm-compare-oracle} reveals, such an improvement relies crucially on the structure of $\var\left[\ell(f_1,z)-\ell(f_2,z)\right]$ in the Bernstein bound. In particular, it cannot be achieved by the na\"{i}ve method of applying \Cref{alg-mean} to $f_1$ and $f_2$ separately and choosing the one with a lower estimated generalization error. 
We believe that in other scenarios such as binary classification, our analysis still goes through under commonly used noise conditions in learning theory \citep{BBM05,BBL05}.

\subsection{Selection from Multiple Candidates}

We now consider the general case of selecting over $m\in\ZZ_+$ models $\{f_r\}_{r=1}^m$. We will use a straightforward single-elimination tournament procedure. In each round, we pair up the remaining models, and use \Cref{alg-compare} to perform pairwise comparison. Within each pair, the model picked by \Cref{alg-compare} advances to the next round. When there is only one model left, the procedure terminates and outputs the model. \Cref{alg-tournament} gives the detailed description.

\begin{algorithm}[t]
	\begin{algorithmic}
	\STATE {\bf Input:} Models $\{f_r\}_{r=1}^m$, data $\{ \datasetB_j \}_{j=1}^t$, hyperparameters $\delta'$ and $M$.
	\STATE Let $S = \lceil \log_2 m \rceil$ and $\functionclass_1=\{f_r:r\in[m]\}$.
	\FOR{$s = 1,\cdots, S$} 
	\STATE Split $\functionclass_s$ into disjoint pairs:
	\[
	\functionclass_s = \cG_{s,1}\cup\cdots\cup\cG_{s,m_s}\cup\cG_{s,m_s+1},
	\]
	where $|\cG_{s,i}|=2$ for $i\in[m_s]$, and $|\cG_{s,m_s+1}|\le 2$.
		\FOR{$i = 1,\cdots, m_s+1$}
		\STATE Run \Cref{alg-compare} with inputs $\cG_{s,i}$, $\{\dataset_j\}_{j=1}^t$, $\delta'$ and $M$ to obtain $\widehat{g}_{s,i}\in\cG_{s,i}$. 
		\STATE If $|\cG_{s,m_s+1}|=1$, simply take $\widehat{g}_{s,m_s+1}\in\cG_{s,m_s+1}$.
		\ENDFOR
	\STATE Let $\functionclass_{s+1} = \{\widehat{g}_{s,i}:i\in[m_s+1]\}$.
	\ENDFOR
	\STATE {\bf Output:} The only model $\widehat{f}\in\functionclass_S$.
	\caption{Single-Elimination Tournament for Model Selection}
	\label{alg-tournament}
	\end{algorithmic}
\end{algorithm}

Here $\functionclass_{s+1}$ is the set of remaining models after round $s$. By design, \Cref{alg-tournament} eliminates about half of the remaining models in each round: $|\functionclass_{s+1}| \le \lceil |\functionclass_s|/2 \rceil$. Thus, only one model remains after $\lceil\log_2m\rceil$ rounds. Since each call of \Cref{alg-compare} eliminates one model, then \Cref{alg-tournament} calls \Cref{alg-compare} exactly $m-1$ times.

We now give the theoretical guarantee of \Cref{alg-tournament} in the setting of bounded regression. We provide a more precise bound and its proof in \Cref{sec-thm-select-oracle-proof}.

\begin{theorem}[Oracle inequality]\label{thm-select-oracle}
Let Assumptions \ref{assumption-bounded-regression}, \ref{assumption-bounded-functions} and \ref{assumption-no-covariate-shift} hold. Let $M_0$ be a constant. Choose $\delta\in(0,1)$ and take $\delta' = \delta / (3m^2t)$ in \Cref{alg-tournament}. With probability at least $1-\delta$, \Cref{alg-tournament} outputs $\widehat{f}$ satisfying
\[
\|\widehat{f}-f_t^*\|_{\distP} - \min_{r\in[m]}\|f_r-f_t^*\|_{\distP} \\
\lesssim 
\min_{k\in[t]}\left\{\max_{t-k+1\le j\le t}\|f_j^*-f_t^*\|_{\distP} + \frac{1}{  \sqrt{B_{t,k}} }\right\}.
\]
Here $\lesssim$ hides a polylogarithmic factor of $m$, $t$ and $\delta^{-1}$.
\end{theorem}

We remark that \Cref{thm-select-oracle} takes the same form as \Cref{thm-compare-oracle} up to a factor of $\log_2m$. Thus, in \Cref{eg-local-stationarity} where $\cP_{t-K}\neq\cP_{t-K+1}=\cdots=\cP_t$, the model $\widehat{f}$ selected by \Cref{alg-tournament} also enjoys a fast rate similar to \eqref{eqn-compare-fast-rate-eg}:
\[
L_t(\widehat{f}) - L_t(f_t^*)
\lesssim
\min_{r\in[m]}L_t(f_r) - L_t(f_t^*)
+
\frac{1}{B_{t,K}}.
\]

%% file: main_experiments.tex
\section{Numerical Experiments}\label{sec-experiments}

We conduct simulations to verify our theories and test our algorithms on three real datasets. We will focus on \Cref{alg-tournament} for Problem \ref{problem-selection}. For simplicity, the hyperparameters $\delta'$ and $M$ in our algorithms are set to be $0.1$ and $0$ throughout our numerical study. For notational convenience, we denote \Cref{alg-tournament} by $\algva_{\rm ARW}$. We compare $\algva_{\rm ARW}$ with fixed-window model selection procedures, given by \Cref{alg-fixed} and denoted by $\algva_k$, where $k$ is the fixed window size and takes values in a set $I=\{1,4,16,64,256\}$. Our code is available at \url{https://github.com/eliselyhan/ARW}.

\begin{algorithm}[t]
	\begin{algorithmic}
		\STATE {\bf Input:} Models $\{f_r\}_{r=1}^m$, data $\{ \datasetB_j \}_{j=1}^t$, window size $k$.
		\STATE Compute $s=\max\{t,k\}$.
		\FOR{$r = 1,\cdots, m$}
		\STATE $\widehat{L}_{t,k}(f_r) = \frac{1}{B_{t,s}}\sum_{j=t-s+1}^t\sum_{i=1}^{B_j}\loss(f_r,z_{j,i})$.
		\ENDFOR
		\STATE Compute $\widehat{r} \in\argmin_{r\in[m]}\widehat{L}_{t,k}(f_r)$.
		\STATE {\bf Output:} $\widehat{f}_k = f_{\widehat{r}}$.
		\caption{Fixed-Window Model Selection Algorithm} 
		\label{alg-fixed}
	\end{algorithmic}
\end{algorithm}

\subsection{Synthetic Data}
We first train different models for a mean estimation task on synthetic data. Then, we deploy \Cref{alg-tournament,alg-fixed} to select top-performing models. Finally, we assess and compare their qualities. 
Throughout the simulations, we consider 100 time periods. At each time period $t$, we generate a batch of $B_t\in\ZZ_+$ i.i.d.~samples $\datasetB_t = \{ z_{t, i} \}_{i=1}^{B_t}$. In our setup, each sample $z_{t, i} \sim \mathcal N(\mu_t, \sigma^2)$, where $\mu_t$ is the given population mean of period $t$. The task of the models is to estimate $\mu_t$.

At each period $t$, we split $\datasetB_t$ into a training set $\datasetB_t^{\rm tr}$ and a validation set $\datasetB_t^{\rm va}$. The models are trained on the training set, and subsequent model selection is done using the validation set. The size of the validation set $B_t^{\rm va}$ is sampled uniformly from $\{2,3,4\}$. The size of the training set is $B_t^{\rm tr} = 3B_t^{\rm va}$. 

In each period $t$, we consider $5$ estimates $\{\algtr(t,w)\}_{w\in I}$ of the target $\mu_t$, where $\algtr(t,w)$ is the sample average of the data $\{\datasetB_j^{\rm tr}\}_{j=t-w+1}^t$. We then select a model based on data $\{\datasetB_j^{\rm va}\}_{j=1}^t$. To measure its quality, we compute the excess risk, which is the squared difference between the true mean $\mu_t$ and the selected estimate. We compare $\algva_{\rm ARW}$ with the fixed-window benchmarks $\{\algva_k\}_{k\in I}$. We use the following scenarios as testbeds.

\begin{figure}[h]
	\centering
    \begin{subfigure}[\Cref{eg-syn-1}]{
	\label{fig-horizontal-true-means}
    \includegraphics[scale=0.6]{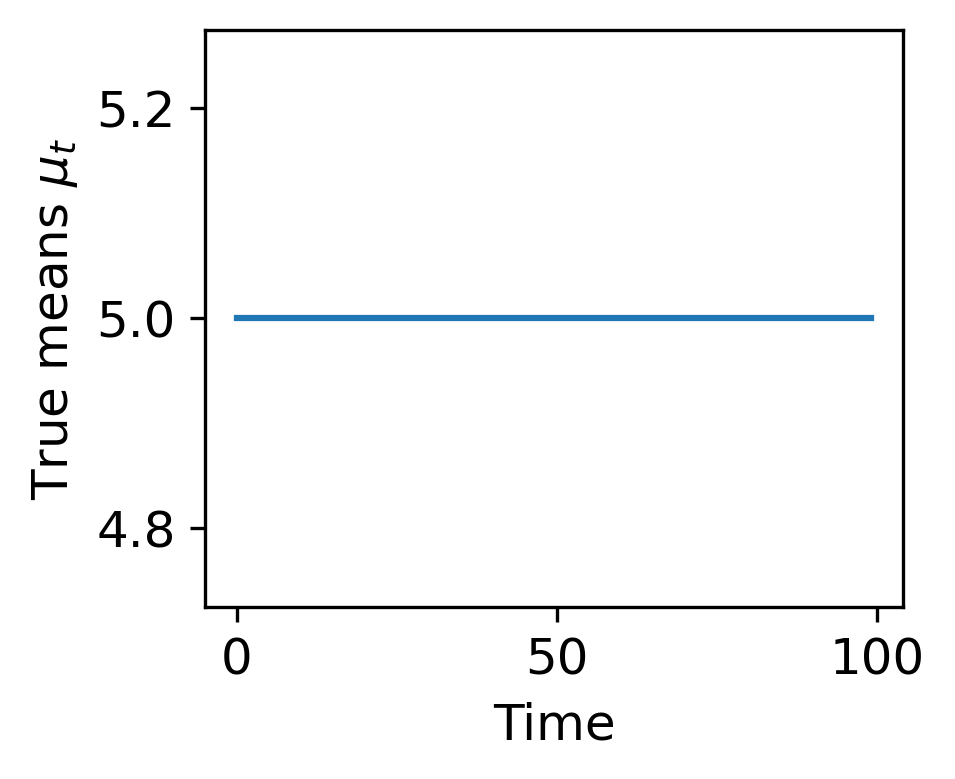}
    }
    \end{subfigure}
    \quad
	\begin{subfigure}[\Cref{eg-syn-2}]{
	\label{fig-combination-true-means}
    \includegraphics[scale=0.6]{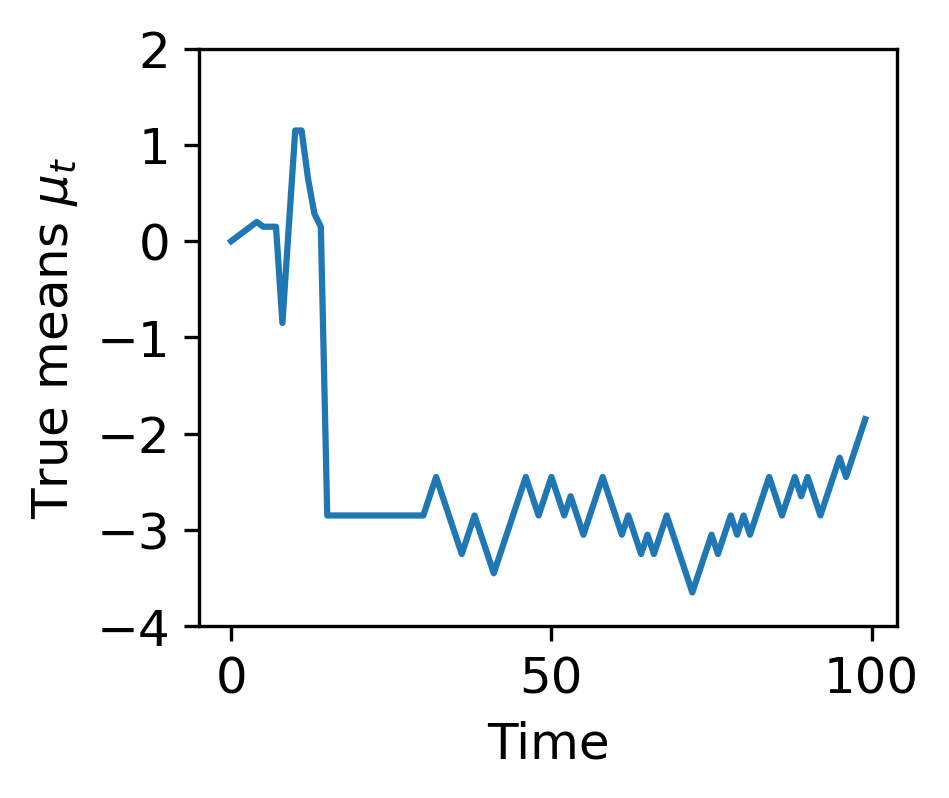}
    }
    \end{subfigure}
    \caption{True means $\{ \mu_t \}_{t=0}^{100}$ in the synthetic data.}
    \label{fig:1}
\end{figure}

\begin{example}\label{eg-syn-1} 
\Cref{fig-horizontal-true-means} illustrates the stationary case where the true mean stays constant. We consider both low-variance and high-variance regimes with $\sigma^2=1$ and $10$, respectively. \Cref{table-eg1} records the average excess risks over $100$ periods and $20$ independent trials for $\algva_{\rm ARW}$ and $\{ \algva_k \}_{k \in I }$. The first and second rows correspond to $\sigma^2=1$ and $\sigma^2=10$. In both regimes, $\algva_{\rm ARW}$ leverages the underlying stationarity and yields excess risks comparable to $\algva_k$ with large $k$'s, whereas $\algva_k$ with smaller $k$'s perform worse. 
In \Cref{fig-eg1}, we plot the average excess risks over $20$ trials at each time $t$ for $\algva_{\rm ARW}$, $\algva_1$ and $\algva_{256}$.
\end{example}

\begin{table}[h]
	\caption{Mean excess risks of different model selection methods in \Cref{eg-syn-1}.}
	\label{table-eg1}
	\vskip 0.15in
	\begin{center}
		\begin{small}
			\begin{sc}
				\begin{tabular}{lcccccr}
					\toprule
					 $\algva_{\rm ARW}$ & $\algva_1$ & $\algva_4$ & $\algva_{16}$ & $\algva_{64}$ & $\algva_{256}$ \\
					\midrule
					 0.015 &  0.043   & 0.025 & 0.013 & 0.010  & 0.010  \\
				  1.293 & 4.117 & 2.572 & 1.396 & 1.015 & 0.982  \\
					\bottomrule
				\end{tabular}
			\end{sc}
		\end{small}
	\end{center}
	\vskip -0.1in
\end{table}

\begin{figure}[h]
    \begin{subfigure}{
    \includegraphics[scale=0.5]{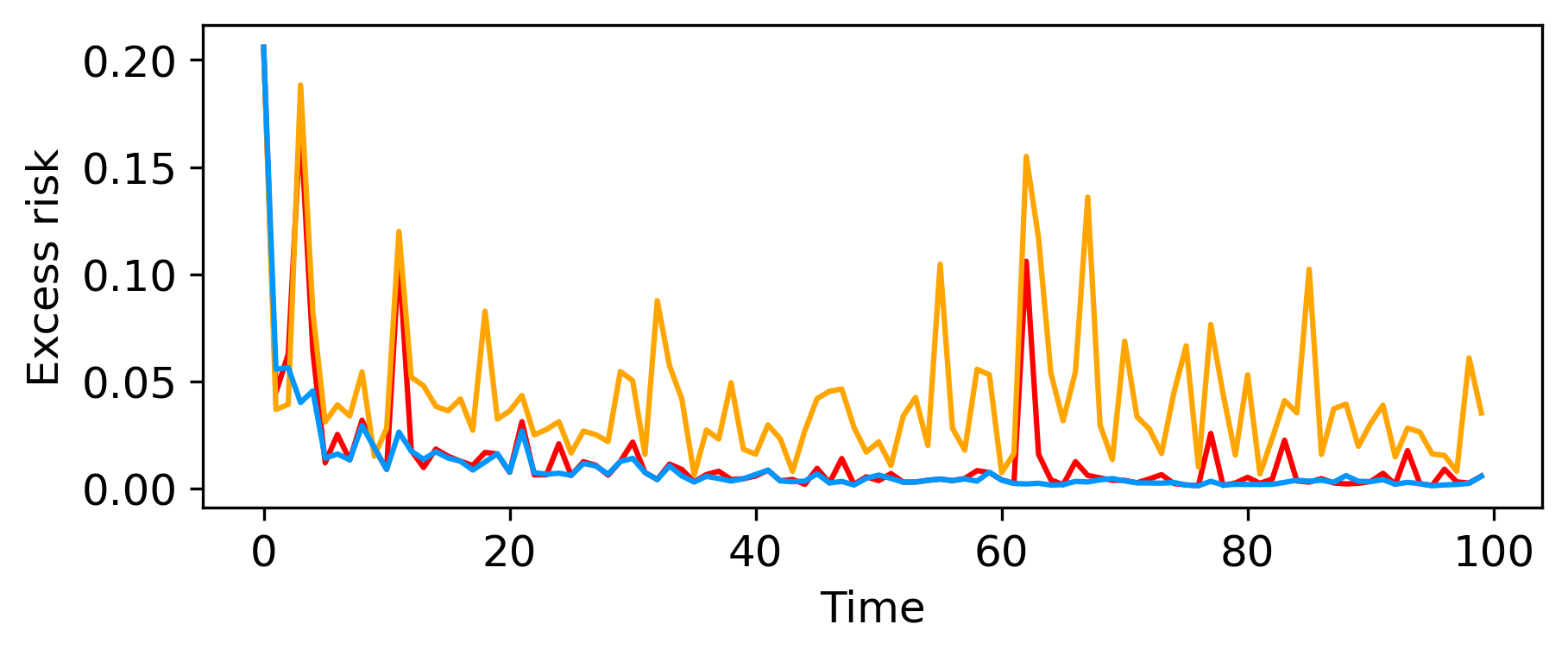}
    }
    \end{subfigure}
	\begin{subfigure}{
    \includegraphics[scale=0.5]{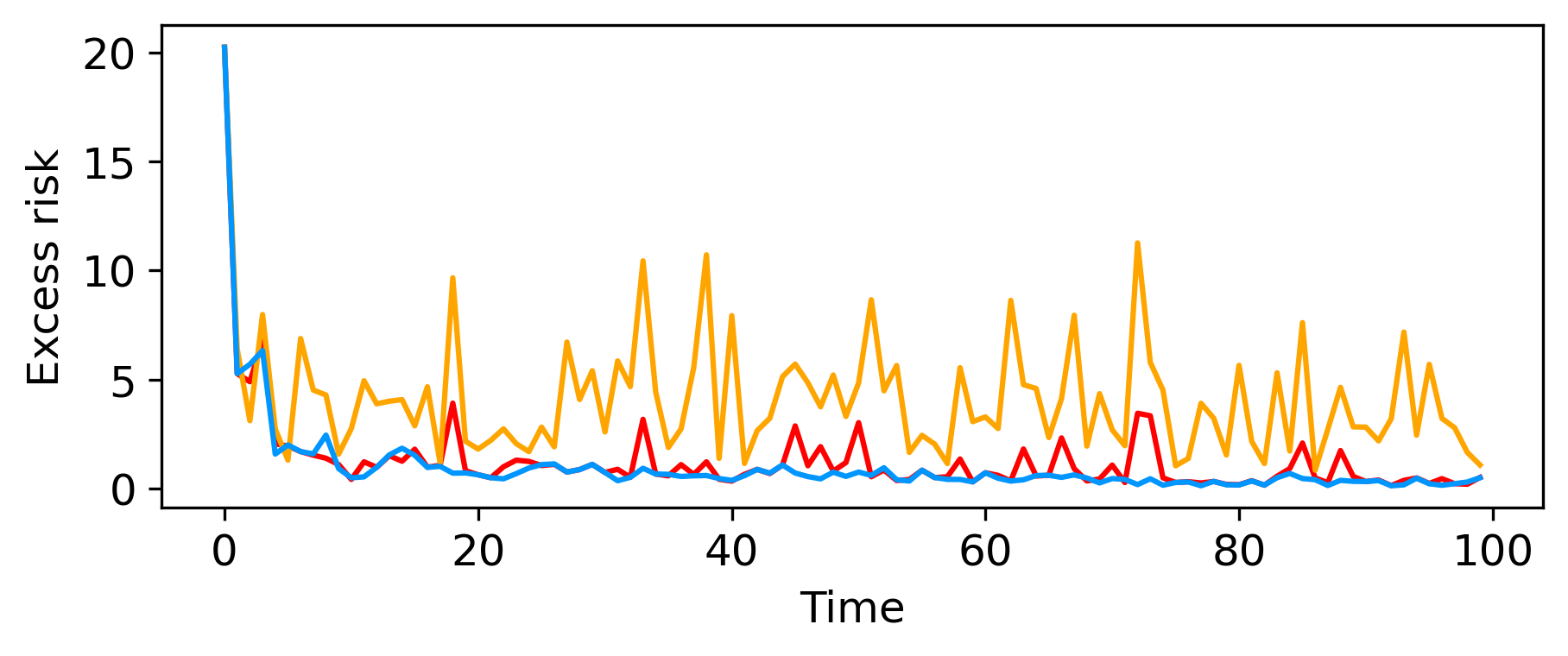}
    }
    \end{subfigure}
    \caption{Excess risks of different model selection methods in \Cref{eg-syn-1}. Left: $\sigma^2 = 1$. Right: $\sigma^2 = 10$. Red: $\algva_{\rm ARW}$. Orange: $\algva_1$. Blue: $\algva_{256}$.}
    \label{fig-eg1}
\end{figure}

\begin{example}\label{eg-syn-2}
We carry out the same experiments in a scenario with sufficient non-stationarity. The true means $\{\mu_t\}_{t=1}^{100}$ in \Cref{fig-combination-true-means} are generated using a randomized mechanism; see \Cref{sec-curve} for the details. Similar to that of \Cref{eg-syn-1}, \Cref{table-eg2} summarizes the mean excess risks. When $\sigma^2=1$, the non-stationarity pattern of the underlying means is largely preserved. Our algorithm $\algva_{\rm ARW}$ outperforms all $\algva_k$'s, demonstrating its adaptivity. When $\sigma^2=10$, the larger noise makes the non-stationarity less significant, so $\algva_k$ with larger $k$'s are be more stable. Nevertheless, $\algva_{\rm ARW}$ is still competitive with them. We also plot the average excess risks over $20$ trials at each time $t$ in \Cref{fig-eg2}. 
\end{example}

\begin{table}[h]
	\caption{Mean excess risks of different model selection methods in \Cref{eg-syn-2}.}
	\label{table-eg2}
	\vskip 0.15in
	\begin{center}
		\begin{small}
			\begin{sc}
				\begin{tabular}{lcccccr}
					\toprule
					 $\algva_{\rm ARW}$ & $\algva_1$ & $\algva_4$ & $\algva_{16}$ & $\algva_{64}$ & $\algva_{256}$ \\
					\midrule
					 0.139 &  0.157   & 0.171 & 0.539 & 1.034  & 1.067  \\
				  2.052 & 4.425 & 2.934 & 1.920 & 1.771 & 1.784  \\
					\bottomrule
				\end{tabular}
			\end{sc}
		\end{small}
	\end{center}
	\vskip -0.1in
\end{table}

\begin{figure}[h]
    \begin{subfigure}{
    \includegraphics[scale=0.5]{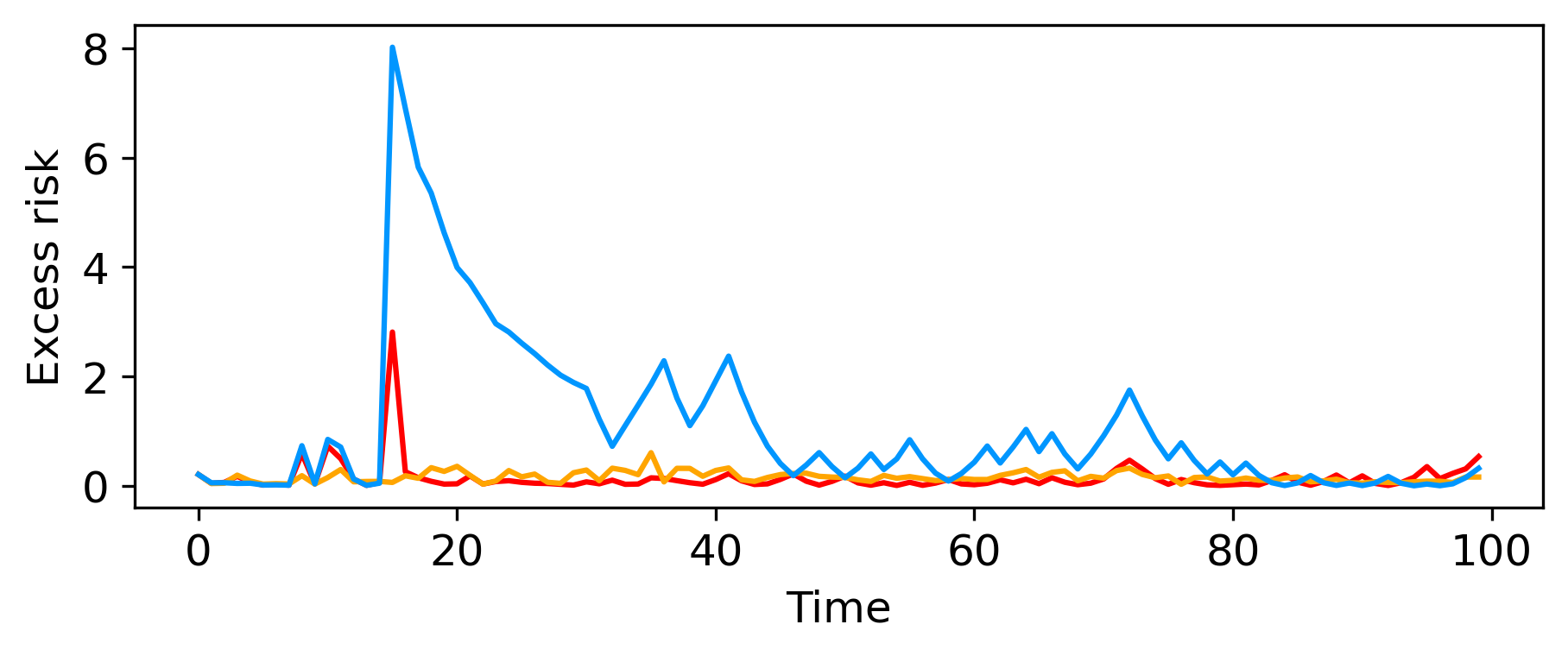}
    }
    \end{subfigure}
	\begin{subfigure}{
    \includegraphics[scale=0.5]{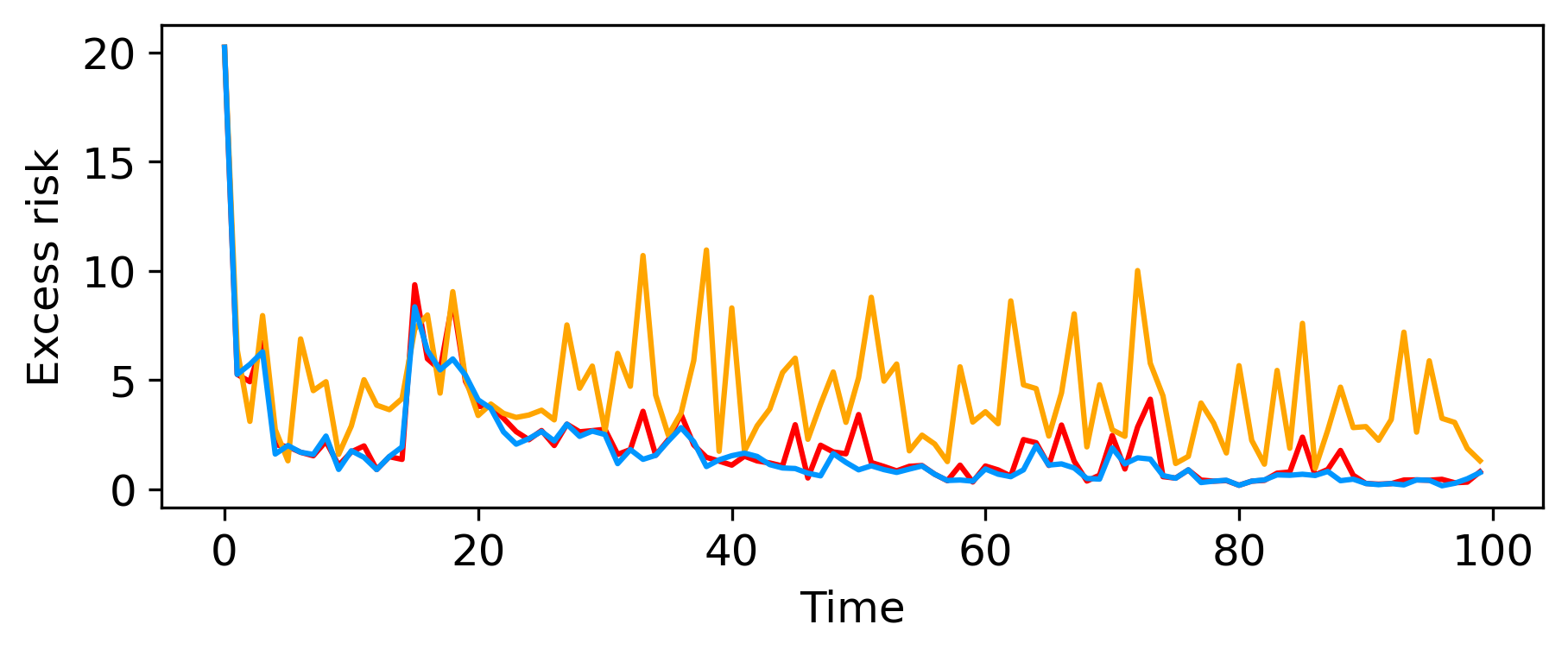}
    }
    \end{subfigure}
    \caption{Excess risks of different model selection methods in \Cref{eg-syn-2}. Left: $\sigma^2 = 1$. Right: $\sigma^2 = 10$. Red: $\algva_{\rm ARW}$. Orange: $\algva_1$. Blue: $\algva_{256}$.}
    \label{fig-eg2}
\end{figure}

\subsection{Real Data: Topic Frequency Estimation}

The first real dataset we use is the arXiv dataset\footnote{\url{https://www.kaggle.com/datasets/Cornell-University/arxiv}}. It consists of basic information of papers submitted to \texttt{arXiv.org}, such as title, abstract and categories. We study the topics of the papers in the category $\texttt{cs.LG}$ from Match $1$st, $2020$ to December $31$st, $2023$. There are $118,883$ papers in total, and we would like to estimate the proportion of deep learning papers in this category. Each week is a time period, and there are $200$ weeks in total. We regard a paper as a deep learning paper if its abstract contains at least one of the words ``deep'', ``neural'', ``dnn'' and ``dnns''. 
The data in each period is randomly split into training, validation and test sets. The training set $\datasetB_j^{\rm tr}$ has $15$ samples, the validation set $\datasetB_j^{\rm va}$ has $5$ samples and the rest of the samples $\datasetB_j^{\rm te}$ are used for testing. Typically $|\datasetB_j^{\rm te}|\in[300,1200]$. In \Cref{fig-arxiv-true-means}, we plot the frequencies over $200$ weeks estimated from $\{\datasetB_j^{\rm te}\}_{j=1}^{200}$, which exhibits a slowly drifting pattern.

In each period $t$, we consider $5$ models $\{\algtr(t,w)\}_{w\in I}$, where $\algtr(t,w)$ computes the average frequency from the data $\{\datasetB_j^{\rm tr}\}_{j=t-w+1}^t$. The selection and evaluation procedures are similar to those in the synthetic data experiment.
We compare $\algva_{\rm ARW}$ with the fixed-window benchmarks $\{\algva_k\}_{k\in I}$. The average excess risks over 200 weeks and 20 independent runs are listed in \Cref{table-arxiv}. In \Cref{fig-arxiv-line-few}, we also plot the excess risks in every period. We observe that the performance of $\algva_{\rm ARW}$ is comparable to that of the large-window benchmark $\algva_{256}$, while the small-window model selection algorithm $\algva_1$ performs poorly.
\begin{table}[h]
	\caption{Mean excess risks ($\times 10^{-3}$) of different model selection methods on the arXiv data.}
	\label{table-arxiv}
	\vskip 0.15in
	\begin{center}
		\begin{small}
			\begin{sc}
				\begin{tabular}{lcccccr}
					\toprule
					 $\algva_{\rm ARW}$ & $\algva_1$ & $\algva_4$ & $\algva_{16}$ & $\algva_{64}$ & $\algva_{256}$ \\
					\midrule
				  2.4 & 6.7 & 4.5 & 2.4 & 1.7 & 1.9  \\
					\bottomrule
				\end{tabular}
			\end{sc}
		\end{small}
	\end{center}
	\vskip -0.1in
\end{table}

\begin{figure}[h]
	\centering
    \includegraphics[scale=0.55]{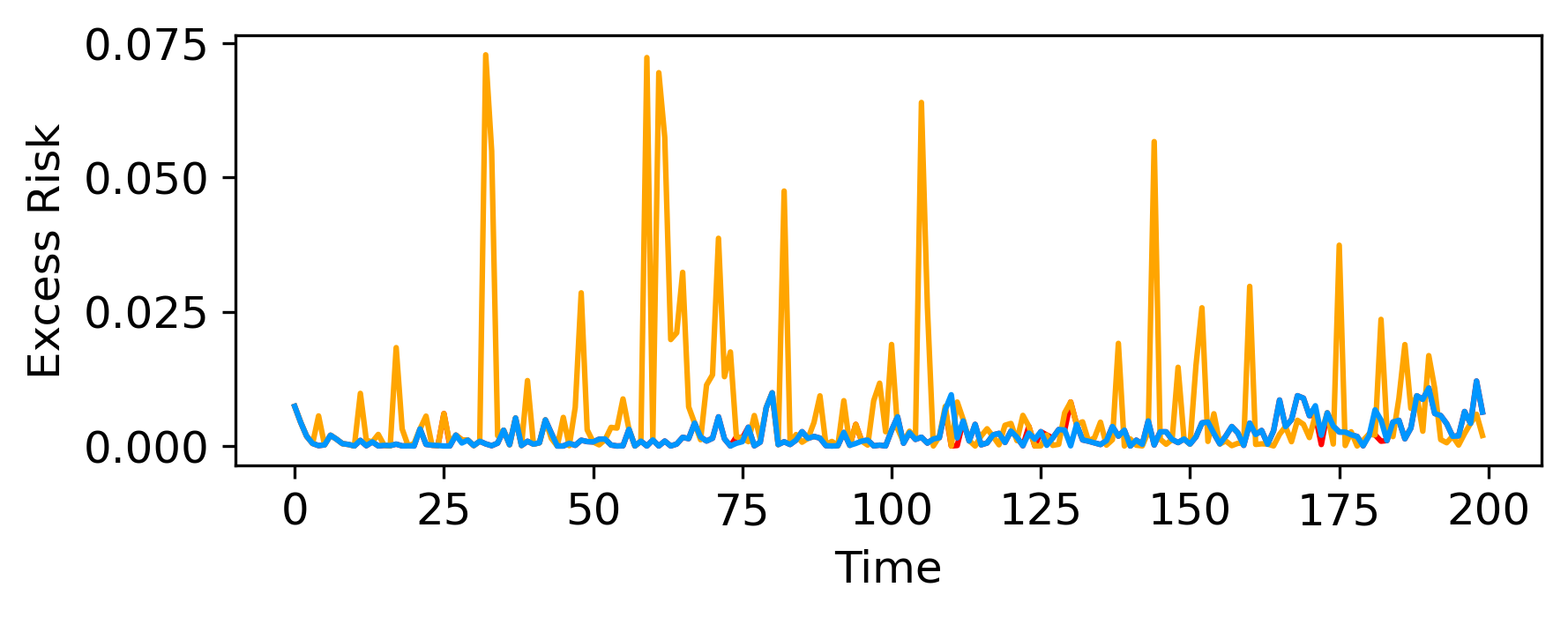}
    \caption{Error curves of different model selection methods on the arXiv data. Red: $\algva_{\rm ARW}$. Orange: $\algva_1$. Blue: $\algva_{256}$.}
    \label{fig-arxiv-line-few}
\end{figure}

\subsection{Real Data: House Price Prediction}

Finally, we test our method using a real-world dataset maintained by the Dubai Land Department\footnote{\url{https://www.dubaipulse.gov.ae/data/dld-transactions/dld_transactions-open}}. 
We study sales of apartments during the past 16 years (from January 1st, 2008 to December 31st, 2023). 
There are 211,432 samples (sales) in total, and we want to predict the final price given characteristics of an apartment (e.g.,~number of rooms,~size,~location). Each month is treated as a time period, and there are 192 of them in total. Our goal is to build a prediction model for each period using historical data. In \Cref{fig-housing-true-means}, we plot the monthly average prices. Compared with the arXiv dataset, the distribution shift in this case is more abrupt.

The data in each period is randomly split into training, validation and test sets with proportions 60\%, 20\% and 20\%, respectively. We follow the standard practice to apply a logarithmic transform to the price and target that in our prediction. See \Cref{sec-housing-details} for details of preprocessing.

\begin{table}[h]
	\caption{Overall MSE of different model selection methods on the housing data.}
	\label{table:housing}
	\vskip 0.15in
	\begin{center}
		\begin{small}
			\begin{sc}
				\begin{tabular}{lcccccr}
					\toprule
					$\algva_{\rm ARW}$ & $\algva_1$ & $\algva_4$ & $\algva_{16}$ & $\algva_{64}$ & $\algva_{256}$ \\
					\midrule
					0.071 &  0.071   & 0.069 & 0.071 & 0.091  & 0.095  \\
					\bottomrule
				\end{tabular}
			\end{sc}
		\end{small}
	\end{center}
	\vskip -0.1in
\end{table}

For the $t$-th period, we use each of the 5 training windows $w \in I$ to select training data $\{ \datasetB_j^{\rm tr} \}_{j=t-w+1}^{t}$ for regression with 2 algorithms: random forest \citep{Bre01} and XGBoost \citep{CGu16}. This results in 10 candidate models. One of them is selected using validation data and finally evaluated on the test data $\datasetB_t^{\rm te}$ by the mean squared error (MSE). For each model selection method, we compute the average MSE over all of the 192 time periods and 20 independent runs. We compare our proposed approach with fixed-window benchmarks $\{ \algva_k \}_{k \in I}$. The mean values are reported in Table \ref{table:housing}. In addition, we also plot the test MSEs of our method, $\algva_1$ and $\algva_{256}$ in all individual time periods, see \Cref{fig-housing}. Due to the strong non-stationarity, model selection based on large windows does not work well. Our method still nicely adapts to the changes.

\begin{figure}[h]
	\vskip 0.2in
	\begin{center}
		\includegraphics[scale=0.55]{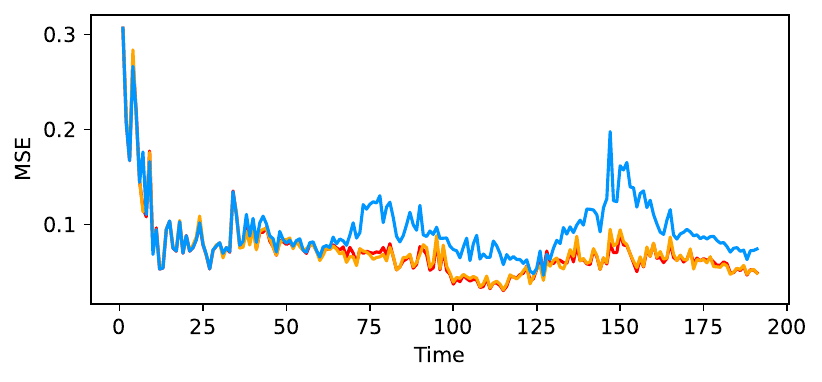}
		\caption{Error curves of different model selection methods on the housing data. Red: $\algva_{\rm ARW}$. Orange: $\algva_1$. Blue: $\algva_{256}$.}
		\label{fig-housing}
	\end{center}
	\vskip -0.2in
\end{figure}

\subsection{Summary of Experiments}

The synthetic and real data experiments above show that for different patterns of non-stationarity, the best window can be different. It is large for stationary environments, medium for the arXiv data, and small for the Dubai housing data. In practice, the non-stationarity pattern is generally unknown, so it is not clear \emph{a priori} what the best window should be, or even what candidate windows to choose from. Our experiments show that our algorithm \emph{adaptively} selects a window that is comparable to the best one \emph{in hindsight}.

%% file: main_discussions.tex
\section{Discussions}\label{sec-discussions}

We developed adaptive rolling window approaches to model assessment and selection in a changing environment. Theoretical analyses and numerical experiments demonstrate their adaptivity to unknown temporal distribution shift. Several future directions are worth exploring. First, our rolling window methodology does not assume any structures on the non-stationarity. In practice, temporal distribution shift often exhibits regularities such as seasonalities and trends. It is important to develop methods that automatically detect patterns of distribution shift and make use of data from the more distant past. Second, our model selection algorithm only applies to finitely many models, and it would be interesting to extend it to infinite model classes. Finally, another direction is to use our model selection algorithm to perform hyperparameter tuning for online learning algorithms, so as to further boost their performance.

%% file: appendix_proof_assess.tex
\section{Proofs for \Cref{sec-assess}}

\subsection{Proof of \Cref{lem-Hoeffding-Bernstein}}\label{sec-lem-Hoeffding-Bernstein-proof}

Inequality (2.10) in \cite{BLM13} implies that for any $t \geq 0$,
\begin{align*}
& \PP \bigg(
\frac{1}{n} \sum_{i=1}^{n} ( x_i - \EE x_i ) > t
\bigg) \le \exp \left(
- \frac{ n t^2 / 2 }{ \sd^2 + (b - a) t / 3 }
\right) .
\end{align*}
Fix $\delta \in (0, 1)$. Then,
\begin{align*}
&\quad \exp  \left(
- \frac{ n t^2 / 2 }{ \sd^2 + (b-a) t / 3 }
\right) \leq \delta \\[4pt]
 \Leftrightarrow &\quad \frac{n  t^2 }{2} \geq \sd^2 \log (1 / \delta) + \frac{t  (b-a) \log ( 1 / \delta ) }{ 3 } \\[4pt]
 \Leftrightarrow &\quad \frac{n}{2} \bigg(
t - \frac{  (b-a) \log (1 / \delta) }{ 3n }
\bigg)^2 \geq \sd^2 \log (1 / \delta)  + \frac{n}{2} \bigg( \frac{  (b-a) \log (1 / \delta) }{3 n } \bigg)^2 \\[4pt]
 \Leftarrow &\quad \bigg(
t - \frac{  (b-a) \log (1 / \delta) }{ 3n }
\bigg)^2 \geq  \bigg( \sd \sqrt{ \frac{ 2 \log (1 / \delta)  }{ n } } + \frac{  (b-a) \log (1 / \delta) }{3 n } \bigg)^2 \\[4pt]
 \Leftarrow &\quad t \geq \sd \sqrt{ \frac{ 2 \log (1 / \delta)  }{ n } } + \frac{ 2  (b-a) \log (1 / \delta) }{3 n } .
\end{align*}
Hence,
\[
\PP \left(
\frac{1}{n} \sum_{i=1}^{n} ( x_i - \EE x_i ) >
\sd \sqrt{ \frac{ 2 \log (1 / \delta)  }{ n } } + \frac{ 2  (b-a) \log (1 / \delta) }{3 n } 
\right) \leq  \delta.
\]
Replacing each $x_i$ by $-x_i$ gives bounds on the lower tail and the absolute deviation.

\subsection{Proof of \Cref{lem-Bernstein-variance}}\label{sec-lem-Bernstein-variance-proof}

The result is trivial when $B_{t, k} = 1$ (i.e.~$k = B_t = 1$), as $\widehat\psi (t , k , \delta) = \psi (t , k , \delta)$. From now on, we assume that $B_{t, k} \ge 2$. We first present a useful lemma.

\begin{lemma}\label{lem-concentration-variance}
	Let $\{ x_i \}_{i=1}^n $ be independent, $[0, 1]$-valued random variables. Define the sample mean $\meanhat = \frac{1}{n} \sum_{i=1}^{n} x_i$ and the sample variance $\totalsdhat^2 = \frac{1}{n-1} \sum_{i=1}^{n} (x_i - \meanhat)^2$. Let $\mean = \frac{1}{n} \sum_{i=1}^{n} \EE x_i$ and $\totalsd^2 = \EE \totalsdhat^2$. We have
	\[
	\totalsd^2 = \frac{1}{n-1} \sum_{i=1}^n  ( \EE x_i - \mean)^2 + \frac{1}{n} \sum_{i=1}^n \var(x_i) 
	\]
	and for any $\delta \in (0, 1) $,
	\begin{align*}
	& \PP \bigg(
	\totalsd \leq \totalsdhat + \sqrt{ \frac{ 2 \log(1 / \delta) }{n - 1} }
	\bigg) \geq 1 - \delta 
	\qquad\text{and}\qquad
	\PP \bigg(
	\totalsd \geq \totalsdhat - \sqrt{ \frac{ 2 \log(1 / \delta) }{n - 1} }
	\bigg) \geq 1 - \delta .
	\end{align*}
\end{lemma}

\begin{proof}[Proof of \Cref{lem-concentration-variance}]
	Define $\bx = (x_1,\cdots,x_n)^{\top}$, $\bmu = \EE \bx$ and $\bSigma  =  \diag ( \var(x_1) , \cdots, \var(x_n) )$. Let $\bm{1}_n$ be the $n$-dimensional all-one vector. We have $\meanhat = \bm{1}_n^{\top} \bx / n$ and
	\[
	\sum_{i=1}^n ( x_i - \meanhat )^2 = \| ( \bI - \bm{1}_n \bm{1}_n^{\top} / n ) \bx \|_2^2 = \bx^{\top} ( \bI - \bm{1}_n \bm{1}_n^{\top} / n ) \bx .
	\]
	Then,
	\begin{align*}
	(n-1) \EE \totalsdhat^2 
	& = \langle \bmu \bmu^{\top} + \bSigma , \bI - \bm{1}_n \bm{1}_n^{\top} / n \rangle = \bmu^{\top} ( \bI - \bm{1}_n \bm{1}_n^{\top} / n ) \bmu + \langle  \bSigma , \bI - \bm{1}_n \bm{1}_n^{\top} / n \rangle \\
	& = \sum_{i=1}^n ( \EE x_i - \mean )^2 + \bigg( 1 - \frac{1}{n} \bigg) \sum_{i=1}^n \var (x_i).
	\end{align*}
	This verifies the expression of $\totalsd^2$. The concentration bounds come from Theorem 10 in \cite{MPo09}.
\end{proof}

We now come back to \Cref{lem-Bernstein-variance}. It suffices to consider the special case $b-a = 1$. From \Cref{lem-concentration-variance} we immediately get $	\EE \totalsdhat_{t, k}^2 = \totalsd_{t, k}^2$, where
\begin{align}
\totalsd_{t, k}^2
= \sd^2_{t, k}  + \frac{1}{ B_{t, k} - 1}  \sum_{j=t-k+1}^t B_{j} ( \mean_{j} - \mean_{t, k} )^2 .
\label{eqn-v-k}
\end{align}
In addition, for any $\delta \in (0, 1)$,
\begin{align*}
& \PP \left(
\totalsd_{t, k} \leq \totalsdhat_{t, k} + \sqrt{ \frac{ 2 \log(1 / \delta) }{ B_{t, k} - 1} }
\right) \geq 1 - \delta 
\qquad\text{and}\qquad
\PP \left(
\totalsd_{t, k} \geq \totalsdhat_{t, k} - \sqrt{ \frac{ 2 \log(1 / \delta) }{ B_{t, k} - 1} }
\right) \geq 1 - \delta .
\end{align*}

With probability at least $1 - \delta$, we have $\sd_{t, k} \leq \totalsd_{t, k} \leq  \totalsdhat_{t, k} + \sqrt{ \frac{ 2 \log(1 / \delta) }{B_{t, k} - 1} } $ and thus
\begin{align*}
\psi (t , k , \delta ) 
& \leq \left(
\totalsdhat_{t, k} + \sqrt{ \frac{ 2 \log(1 / \delta) }{B_{t, k} - 1} }
\right) \sqrt{ \frac{2 \log( 2 / \delta) }{ B_{t, k} } } + \frac{ 2 \log ( 2 / \delta) }{ 3 ( B_{t, k} - 1 ) } \\[4pt]
& \leq \totalsdhat_{t, k}  \sqrt{ \frac{2 \log( 2 / \delta) }{ B_{t, k} } } + \frac{ 8 \log ( 2 / \delta) }{ 3 ( B_{t, k} - 1 ) } 
= 
\widehat\psi (t , k , \delta ) .
\end{align*}

To prove the second bound, note that
\begin{align*}
\totalsd_{t, k} 
& =\bigg(
\sd^2_{t, k}  + \frac{1}{ B_{t, k} - 1}  \sum_{j=t-k+1}^{t} B_j ( \mean_j - \mean_{t, k} )^2
\bigg)^{1/2}
\\
& \leq \sd_{t, k} + \bigg(
\frac{1}{B_{t, k} - 1}  \sum_{j=t-k+1}^{t} B_j ( \mean_j - \mean_{t, k} )^2
\bigg)^{1/2} \\
& = \sd_{t, k}  + \sqrt{ \frac{B_{t, k}}{B_{t, k} - 1} } \cdot \bigg(
\frac{1}{B_{t, k}} \sum_{j=t-k+1}^{t} B_j ( \mean_j - \mean_{t} )^2
\bigg)^{1/2} \\
& \leq \sd_{t, k} + \sqrt{ \frac{B_{t, k}}{B_{t, k} - 1} } \cdot 
\max_{t-k+1\le j\le t} | \mean_j - \mean_{t} | .
\end{align*}
By \Cref{lem-concentration-variance}, with probability at least $1 - \delta$, we have $ \totalsd_{t, k} \geq  \totalsdhat_{t, k} - \sqrt{ \frac{ 2 \log(1 / \delta) }{B_{t, k} - 1} } $ and thus
\begin{align*}
\totalsdhat_{t, k} 
& \leq \totalsd_{t, k} + \sqrt{ \frac{ 2 \log(1 / \delta) }{B_{t, k} - 1} }  \leq  \sd_{t, k} + \sqrt{ \frac{B_{t, k}}{B_{t, k} - 1} } \cdot 
\max_{t-k+1\le j\le t} | \mean_j- \mean_{t} | + \sqrt{ \frac{ 2 \log(1 / \delta) }{B_{t, k} - 1} }  .
\end{align*}
As a result,
\begin{align*}
\psi ( t, k, \delta ) 
& \ge \sd_{t, k} \sqrt{ \frac{2 \log( 2 / \delta) }{ B_{t,k} } } + \frac{  \log ( 2 / \delta) }{ 3 ( B_{t, k} - 1 ) } \\[4pt]
& \geq
\bigg(
\totalsdhat_{t, k}  -  \sqrt{ \frac{B_{t, k}}{B_{t, k} - 1} } \cdot 
\max_{t-k+1\le j\le t} | \mean_j - \mean_{t} | - \sqrt{ \frac{ 2 \log(1 / \delta) }{B_{t, k} - 1} }  
\bigg)
\sqrt{ \frac{2 \log( 2 / \delta) }{ B_{t, k} } } + \frac{ \log ( 2 / \delta) }{ 3 ( B_{t, k} - 1 ) }  \\[4pt]
& \geq 
\totalsdhat_{t, k} \sqrt{ \frac{2 \log( 2 / \delta) }{ B_{t, k} } } - 
\sqrt{ \frac{ 2 \log ( 2 / \delta ) }{B_{t, k} - 1} } \cdot 
\max_{t-k+1\le j\le t} | \mean_j - \mean_{t} | -  \frac{ 5 \log(2 / \delta) }{ 3 ( B_{t, k} - 1) } 
\\[4pt]
& = \bigg(
\totalsdhat_{t, k}  \sqrt{ \frac{2 \log( 2 / \delta) }{ B_{t, k} } } + \frac{ 8 \log ( 2 / \delta) }{ 3 ( B_{t, k} - 1 ) } \bigg) - 
\sqrt{ \frac{ 2 \log ( 2 / \delta ) }{B_{t, k} - 1} } \cdot 
\max_{t-k+1\le j\le t} | \mean_j - \mean_{t} | -  \frac{ 13 \log(2 / \delta) }{   3(B_{t, k} - 1)  } \\[4pt]
& = \widehat\psi ( t , k , \delta) - 
\sqrt{ \frac{ 2 \log ( 2 / \delta ) }{B_{t, k} - 1} } \cdot 
\max_{t-k+1\le j\le t} | \mean_j - \mean_{t} | -  \frac{ 13 \log(2 / \delta) }{  3(B_{t, k} - 1)  } \\[4pt]
& \geq \widehat\psi ( t , k , \delta) - \xi (t , k , \delta).
\end{align*}
The first and last inequalities follow from the fact that $B_{t, k} \geq 2$.

\subsection{Proof of \Cref{lem-bias-proxy-vs-bias}}\label{sec-lem-bias-proxy-vs-bias-proof}

When the event \eqref{eqn-assess-bias-variance-empirical} happens, for every $i\in[k]$,
\[
|\widehat{\mean}_{t,k} - \widehat{\mean}_{t,i}|
\le 
|\widehat{\mean}_{t,k} - \mean_t| + |\widehat{\mean}_{t,i}-\mean_t| 
\le 
\left[\phi(t,k) + \widehat{\psi}(t,k,\delta)\right] + \left[\phi(t,i) + \widehat{\psi}(t,i,\delta)\right],
\]
so
\[
|\widehat{\mean}_{t,k} - \widehat{\mean}_{t,i}| - \left[\widehat{\psi}(t,k,\delta)+\widehat{\psi}(t,i,\delta)\right]
\le 
\phi(t,k) + \phi(t,i)
\le 
2\phi(t,k),
\]
where we used $\phi(t,i)\le \phi(t,k)$. Taking maximum over all $i\in[k]$ gives $\widehat{\phi}(t,k,\delta)\le 2\phi(t,k)$.

\subsection{Proof of \Cref{thm-assess-oracle}}\label{sec-thm-assess-oracle-proof}

We will prove that with probability at least $1-\delta$,
\begin{equation}
| \meanhat_{t, \widehat{k} } - \mean_{t} |  
\leq  3\min_{k \in [t]} \left\{
3 \sqrt{ \log ( 6t / \delta )  } \cdot 
\max_{t-k+1\le j\le t} | \mean_{j} - \mean_{t} |
+
\sd_{t, k} \sqrt{ \frac{2 \log( 6t / \delta) }{ B_{t,k} } } + \frac{ 10 (b-a) \log ( 6t / \delta) }{ B_{t, k} }
\right\}.
\label{eqn-assess-oracle-precise}
\end{equation}

By \Cref{lem-GL},
\begin{equation}
\PP\left(
|\meanhat_{t , \widehat{k} } - \mean_t| \leq 3 \min_{k \in [t]} \left\{
	\phi (t , k ) + \widehat\psi(t, k, \delta') 
	\right\}
\right)
\ge 
1-2t\delta'.
\label{eqn-assess-oracle-empirical}
\end{equation}
By the second bound in \Cref{lem-Bernstein-variance},
\begin{equation}
\PP \bigg(
\widehat\psi (t , k , \delta' ) 
\leq  \psi (t , k , \delta' ) + \xi (t , k , \delta' ) ,~\forall k \in [t]
\bigg) \geq 1 - t\delta' .
\label{eqn-assess-empirical-psi-upper-bound}
\end{equation}
By \eqref{eqn-assess-oracle-empirical}, \eqref{eqn-assess-empirical-psi-upper-bound} and the union bound, we obtain that
\[
\PP\left( |\meanhat_{t , \widehat{k} } - \mean_t| 
\leq 
3 \min_{k \in [t]} 
\big\{
\phi (t , k ) + \psi (t , k , \delta' ) + \xi (t , k , \delta' )
\big\} \right)
\ge 
1-3t\delta'
=
1-\delta.
\]
When $B_{t,k}\ge 2$,
\begin{align*}
& 
\phi (t , k ) + \psi (t , k , \delta' ) + \xi (t , k , \delta' )
 \\[4pt]
& \le 
\max_{t-k+1\le j\le t} | \mean_j - \mean_{t} |
+
\sd_{t, k} \sqrt{ \frac{2 \log( 2 / \delta') }{ B_{t,k} } } + \frac{ 2 (b-a) \log ( 2 / \delta') }{ 3 ( B_{t, k} - 1 ) } \\
&\qquad +
\sqrt{  \frac{4 \log ( 2 / \delta' ) }{B_{t, k}} } \cdot 
\max_{t-k+1\le j\le t} | \mean_{j} - \mean_{t} | +  \frac{ 13 (b-a) \log( 2 / \delta') }{ 3( B_{t, k} - 1 ) }  \\[4pt]
& \le
3 \sqrt{ \log ( 2 / \delta' )  } \cdot 
\max_{t-k+1\le j\le t} | \mean_{j} - \mean_{t} |
+
\sd_{t, k} \sqrt{ \frac{2 \log( 2 / \delta') }{ B_{t,k} } } + \frac{ 5 (b-a) \log ( 2 / \delta') }{ B_{t, k} - 1 } \\[4pt]
&\le 
3 \sqrt{ \log ( 2 / \delta' )  } \cdot 
\max_{t-k+1\le j\le t} | \mean_{j} - \mean_{t} |
+
\sd_{t, k} \sqrt{ \frac{2 \log( 2 / \delta') }{ B_{t,k} } } + \frac{ 10 (b-a) \log ( 2 / \delta') }{ B_{t, k} } \\[4pt]
&=
3 \sqrt{ \log ( 6t / \delta )  } \cdot 
\max_{t-k+1\le j\le t} | \mean_{j} - \mean_{t} |
+
\sd_{t, k} \sqrt{ \frac{2 \log( 6t / \delta) }{ B_{t,k} } } + \frac{ 10 (b-a) \log ( 6t / \delta) }{ B_{t, k} }.
\end{align*}
When $B_{t,k}=1$, we have $\phi (t , k ) + \psi (t , k , \delta' )
=
b-a$, which is dominated by the bound above.

%% file: appendix_proof_select.tex
\section{Proofs for \Cref{sec-select}}

\subsection{Proof of \Cref{lem-from-assess-to-compare}}\label{sec-lem-from-assess-to-compare-proof}
Fix $k\in[t]$. For every $f\in\{f_1,f_2\}$,
\begin{align*}
L_t(f_{\widehat{r}}) - L_t(f)
&=
\big[L_t(f_{\widehat{r}_k}) - \widehat{L}_{t,k}(f_{\widehat{r}_k})\big] + \big[\widehat{L}_{t,k}(f_{\widehat{r}_k}) - L_t(f)\big] \\[4pt]
&\le 
\big[L_t(f_{\widehat{r}_k}) - \widehat{L}_{t,k}(f_{\widehat{r}_k})\big] + \big[\widehat{L}_{t,k}(f) - L_t(f)\big] \\[4pt]
&=
\big[\widehat{L}_{t,k}(f) - \widehat{L}_{t,k}(f_{\widehat{r}_k})\big] - \big[L_t(f) - L_t(f_{\widehat{r}_k})\big] \\[4pt]
&\le 
\left|\big[\widehat{L}_{t,k}(f_1) - \widehat{L}_{t,k}(f_2)\big] - \big[L_t(f_1) - L_t(f_2)\big]\right|.
\end{align*}
Taking minimum over $f\in\{f_1,f_2\}$ yields
\[
L_t(f_{\widehat{r}}) - \min_{r\in[2]}L_t(f_r)
\le 
\left|\big[\widehat{L}_{t,k}(f_1) - \widehat{L}_{t,k}(f_2)\big] - \big[L_t(f_1) - L_t(f_2)\big]\right|.
\]

\subsection{Proof of \Cref{thm-compare-oracle}}\label{sec-thm-compare-oracle-proof}

More precisely, we will prove that with probability at least $1-\delta$, \Cref{alg-compare} outputs $\widehat{f}$ satisfying
\[
\|\widehat{f} - f_t^*\|_{\distP} - \min_{r\in[2]} \|f_r - f_t^*\|_{\distP}
\le C'\sqrt{\log(6t/\delta)}\min_{k\in[t]}\left\{\max_{t-k+1\le j\le t}\|f_j^*-f_t^*\|_{\distP} + \frac{M}{\sqrt{B_{t,k}}}\right\},
\]
where $C'>0$ is a universal constant.

Following the notation in Problem \ref{problem-mean} with $u_{j,i} = \loss(f_1,z_{j,i}) - \loss(f_2,z_{j,i})$, we define
\[
\sd_j^2 = \var(u_{j,1}) = \var\left(\loss(f_1,z_{j,1}) - \loss(f_2,z_{j,1})\right)
\quad\text{and}\quad
\sd_{t,k}^2 = \frac{1}{B_{t,k}}\sum_{j=t-k+1}^tB_j\sd_j^2.
\]
By \Cref{lem-from-assess-to-compare} and \Cref{thm-assess-oracle}, with probability at least $1-\delta$,
\begin{align}
& L_t(\widehat{f}) - \min_{r\in[2]}L_t(f_r) \notag\\
& \le
| \widehat{\Delta}_{t, \widehat{k} } - \Delta_{t} |  \notag \\
& \le
3\min_{k \in [t]} \left\{
3 \sqrt{ \log ( 6t / \delta )  } \cdot 
\max_{t-k+1\le j\le t} | \Delta_{j} - \Delta_{t} |
+
\sd_{t, k} \sqrt{ \frac{2 \log( 6t / \delta) }{ B_{t,k} } } + \frac{ 40M_0^2 \log ( 6t / \delta) }{ B_{t, k} }
\right\}.
\label{eqn-compare-oracle-raw}
\end{align}
From now on suppose that this event happens.

We first derive a bound for $| \Delta_j - \Delta_{t} | $. Since
\[
\Delta_j  = \| f_1 - f_j^* \|_{\distP}^2 - \| f_2 - f_j^* \|_{\distP}^2 
= \langle f_1 - f_2 , f_1 + f_2 - 2 f_j^* \rangle_{\distP},
\]
then for every $j\in\ZZ_+$,
\[
| \Delta_j - \Delta_t | = 2 |
\langle f_1 - f_2 , f_j^* - f_t^* \rangle_{\distP} |
\le 
2\| f_1 - f_2 \|_{\distP} \|  f_j^* - f_t^* \|_{\distP}.
\]
Hence
\begin{equation}\label{eqn-compare-bias-bound}
\max_{t-k+1\le j\le t} | \Delta_j - \Delta_t | 
\le
2\|f_1-f_2\|_{\distP}\max_{t-k+1\le j\le t}\|  f_j^* - f_t^* \|_{\distP}.
\end{equation}

Next, we bound $\sigma_{t,k}$. We have
\begin{align*}
\sd_j^2 
&=
\var \big(
 [ f_1(x_{j, 1}) - y_{j, 1} ]^2 - [ f_2(x_{j, 1}) - y_{j, 1} ]^2
\big) \\[4pt]
&\le 
\EE\left[\big(
 [ f_1(x_{j, 1}) - y_{j, 1} ]^2 - [ f_2(x_{j, 1}) - y_{j, 1} ]^2
\big)^2\right] \\[4pt]
&=
\EE\left[
\big(f_1(x_{j,1})-f_2(x_{j,1})\big)^2
\big(f_1(x_{j,1})+f_2(x_{j,1})-2y_{j,1}\big)^2
\right] \\[4pt]
&\lesssim 
M_0^2\|f_1-f_2\|_{\distP}^2,
\end{align*}
which implies
\begin{equation}
\sd_{t, k} \lesssim
M_0 \| f_1 - f_2 \|_{\distP}
\le 
M_0
\big( 
\| f_1 - f_t^* \|_{\distP} + \| f_2 - f_t^* \|_{\distP} 
\big).
\label{eqn-compare-variance-bound}
\end{equation}

Substituting \eqref{eqn-compare-bias-bound} and \eqref{eqn-compare-variance-bound} into \eqref{eqn-compare-oracle-raw}, there exist a universal constant $C>0$ such that for every $k\in[t]$, 
\begin{align}
L_t(\widehat{f}) - \min_{r\in[2]}L_t(f_r)
&\le 
C\Bigg[
\sqrt{\log(6t / \delta)} \max_{t-k+1\le j\le t}\|f_j^*-f_t^*\|_{\distP} \|f_1-f_2\|_{\distP} \notag\\
&\qquad 
+ M_0\sqrt{\frac{\log(6t / \delta)}{B_{t,k}}} \left( \|f_1-f_t^*\|_{\distP} + \|f_2-f_t^*\|_{\distP} \right)
+ \frac{M_0^2\log(6t / \delta)}{B_{t,k}}
\Bigg] \notag\\[4pt]
& =
C \bigg[
2\Phi(t,k) \left( \|f_1-f_t^*\|_{\distP} + \|f_2-f_t^*\|_{\distP} \right) + \Psi(t,k)
\bigg],
\label{eqn-compare-oracle-denoised}
\end{align}
where
\[
\Phi(t,k) = \frac{1}{2}\sqrt{\log(6t / \delta)}\left(\max_{t-k+1\le j\le t}\|f_j^*-f_t^*\|_{\distP} + \frac{M_0}{\sqrt{B_{t,k}}}\right)
\quad\text{and}\quad
\Psi(t,k) = \frac{M_0^2\log(6t / \delta)}{B_{t,k}},
\]
and we used the triangle inequality $\|f_1-f_2\|_{\distP} \le \|f_1-f_t^*\|_{\distP} + \|f_2-f_t^*\|_{\distP}$. 

Without loss of generality, assume $L_t(f_1)\ge L_t(f_2)$. When $\widehat{r}_{\widehat{k}}=1$, we have $\widehat{f}=f_1$. Then \eqref{eqn-compare-oracle-denoised} yields
\begin{align*}
\|f_1-f_t^*\|_{\distP}^2 - \|f_2-f_t^*\|_{\distP}^2
&=
L_t(\widehat{f}) - \min_{r\in[2]}L_t(f_r) \\[4pt]
&\le 
C \bigg[
2\Phi(t,k) \left( \|f_1-f_t^*\|_{\distP} + \|f_2-f_t^*\|_{\distP} \right) + \Psi(t,k)
\bigg].
\end{align*}
Completing the squares gives
\[
\bigg[\|f_1-f_t^*\|_{\distP}-C\Phi(t,k)\bigg]^2
\le 
\bigg[\|f_2-f_t^*\|_{\distP}+C\Phi(t,k)\bigg]^2
+
C\Psi(t,k).
\]
Hence
\begin{align}
\|f_1-f_t^*\|_{\distP}
&\le 
\|f_2-f_t^*\|_{\distP}+2C\Phi(t,k) + \sqrt{C\Psi(t,k)} \notag\\[4pt]
&\le 
\|f_2-f_t^*\|_{\distP} + C'\sqrt{\log(6t/\delta)}\left(\max_{t-k+1\le j\le t}\|f_j^*-f_t^*\|_{\distP} + \frac{M_0}{\sqrt{B_{t,k}}}\right)
\end{align}
for some universal constant $C'>0$. Since this holds for all $k\in[t]$, we get
\begin{align}
&
\|\widehat{f}-f_t^*\|_{\distP} - \min_{r\in[2]}\|f_r-f_t^*\|_{\distP} \notag\\[4pt]
&=
\|f_1-f_t^*\|_{\distP} - \|f_2-f_t^*\|_{\distP} \notag\\[4pt]
&\le 
C'\sqrt{\log(6t / \delta)}\min_{k\in[t]}\left\{ \max_{t-k+1\le j\le t}\|f_j^*-f_t^*\|_{\distP} + \frac{M_0}{\sqrt{B_{t,k}}} \right\}.
\label{eqn-compare-oracle-unsquared}
\end{align}
We have proved this bound for the case $\widehat{r}_{\widehat{k}}=1$. When $\widehat{r}_{\widehat{k}}=2$, we have $\widehat{f}=f_2$ and hence $\|\widehat{f}-f_t^*\|_{\distP} - \min_{r\in[2]}\|f_r-f_t^*\|_{\distP}=0$, so the bound \eqref{eqn-compare-oracle-unsquared} continues to hold.

\subsection{Proof of \Cref{thm-select-oracle}}\label{sec-thm-select-oracle-proof}

More precisely, we will prove that with probability at least $1-\delta$, \Cref{alg-tournament} outputs $\widehat{f}$ satisfying
\[
\|\widehat{f}-f_t^*\|_{\distP} - \min_{r\in[m]}\|f_r-f_t^*\|_{\distP} 
\le
C'\sqrt{(\log m)\log(mt/\delta)}\min_{k\in[t]}\left\{\max_{t-k+1\le j\le t}\|f_j^*-f_t^*\|_{\distP} + \frac{M_0}{\sqrt{B_{t,k}}}\right\}
\]
for some universal constant $C'>0$.

Denote \Cref{alg-compare} with data $\{\datasetB_j\}_{j=1}^t$ by $\alg$, which takes as input two models $f',f''\in\functionclass$ and outputs the selected model $\alg(\{f',f''\})\in\{f',f''\}$. For notational convenience, we set $\alg(\{f\})=f$ for every $f\in\functionclass$. By \Cref{thm-compare-oracle} and the union bound, with probability at least $1-\delta$, the following holds for all $f',f''\in\{f_r\}_{r=1}^m$:
\begin{multline}
\|\alg(\{f',f''\}) - f_t^*\|_{\distP} - \min_{f\in\{f',f''\}} \|f - f_t^*\|_{\distP} \\
\le C\sqrt{\log(mt/\delta)}\min_{k\in[t]}\left\{\max_{t-k+1\le j\le t}\|f_j^*-f_t^*\|_{\distP} + \frac{M_0}{\sqrt{B_{t,k}}}\right\},
\label{eqn-select-good-selection-over-all-pairs}
\end{multline}
where $C>0$ is a universal constant. From now on, suppose that this event happens.

For each $s\in[S+1]$, let 
\[
g_s\in\argmin_{g\in\functionclass_s}L_t(g)=\argmin_{g\in\functionclass_s}\|g-f_t^*\|_{\distP}.
\] 
Since $\functionclass = \functionclass_1 \supseteq \cdots \supseteq \functionclass_S \supseteq \functionclass_{S+1} = \{\widehat{f}\}$, then
\[
\|\widehat{f}-f_t^*\|_{\distP} - \min_{r\in[m]}\|f_r-f_t^*\|_{\distP} 
=
\sum_{s=1}^S\big(\|g_{s+1}-f_t^*\|_{\distP} - \|g_s-f_t^*\|_{\distP}\big).
\]
For each $s\in[S]$, there exists $i_s\in[m_s+1]$ such that $g_s\in\cG_{s,i_s}$. Let $\widehat{g}_s=\alg(\cG_{s,i_s})$ be the model selected from $\cG_{s,i_s}$ in \Cref{alg-tournament}. By \eqref{eqn-select-good-selection-over-all-pairs} and the definition of $g_s$,
\[
\|\widehat{g}_s - f_t^*\|_{\distP} - \|g_s - f_t^*\|_{\distP}
\le 
C\sqrt{\log(mt/\delta)}\min_{k\in[t]}\left\{\max_{t-k+1\le j\le t}\|f_j^*-f_t^*\|_{\distP} + \frac{M_0}{\sqrt{B_{t,k}}}\right\}.
\]

Moreover, since $\widehat{g}_s\in\functionclass_{s+1}$, then $\|g_{s+1}-f_t^*\|_{\distP} \le \|\widehat{g}_s-f_t^*\|_{\distP}$. Therefore,
\begin{align*}
\|\widehat{f}-f_t^*\|_{\distP} - \min_{r\in[m]}\|f_r-f_t^*\|_{\distP} 
&=
\sum_{s=1}^S\big(\|g_{s+1}-f_t^*\|_{\distP} - \|g_s-f_t^*\|_{\distP}\big) \\[4pt]
&\le 
\sum_{s=1}^S\big(\|\widehat{g}_s-f_t^*\|_{\distP} - \|g_s-f_t^*\|_{\distP}\big) \\[4pt]
&\le 
\sum_{s=1}^S C\sqrt{\log(mt/\delta)}\min_{k\in[t]}\left\{\max_{t-k+1\le j\le t}\|f_j^*-f_t^*\|_{\distP} + \frac{M_0}{\sqrt{B_{t,k}}}\right\} \\[4pt]
&\le
C'\sqrt{(\log m)\log(mt/\delta)}\min_{k\in[t]}\left\{\max_{t-k+1\le j\le t}\|f_j^*-f_t^*\|_{\distP} + \frac{M_0}{\sqrt{B_{t,k}}}\right\}
\end{align*}
for some universal constant $C'>0$.

%% file: appendix_experiments.tex
\section{Numerical Experiments: Additional Details}

\subsection{Example 2 of the Synthetic Data}\label{sec-curve}
We give an outline of how the true mean sequence $\{\mu_t\}$ in \Cref{eg-syn-2} is generated. The sequence is constructed by combining 4 parts, each representing a distribution shift pattern. In the first part, the sequence experiences large shifts. Then, it switches to a sinusoidal pattern. Following that, the environment stays stationary for some time. Finally, the sequence drifts randomly at every period, where the drift sizes are independently sampled from $\{1,-1\}$ with equal probability and scaled with a constant. 

The function for generating the sequence takes in 3 parameters $\texttt{N}$, $\texttt{n}$ and $\texttt{seed}$, where $\texttt{N}$ is the total number of periods, $\texttt{n}$ is the parameter determining the splitting points of the 4 parts, and $\texttt{seed}$ is the random seed used for code reproducibility. In our experiment, we set $\texttt{N = 100}$,\ $\texttt{n = 2}$ and $\texttt{seed = 2024}$. The exact function can be found in our code at \url{https://github.com/eliselyhan/ARW}.

\subsection{Non-Stationarity in arXiv and Housing Data}

\begin{figure}[H]
	\centering
	\begin{subfigure}[arXiv data]{
			\label{fig-arxiv-true-means}
			\includegraphics[scale=0.5]{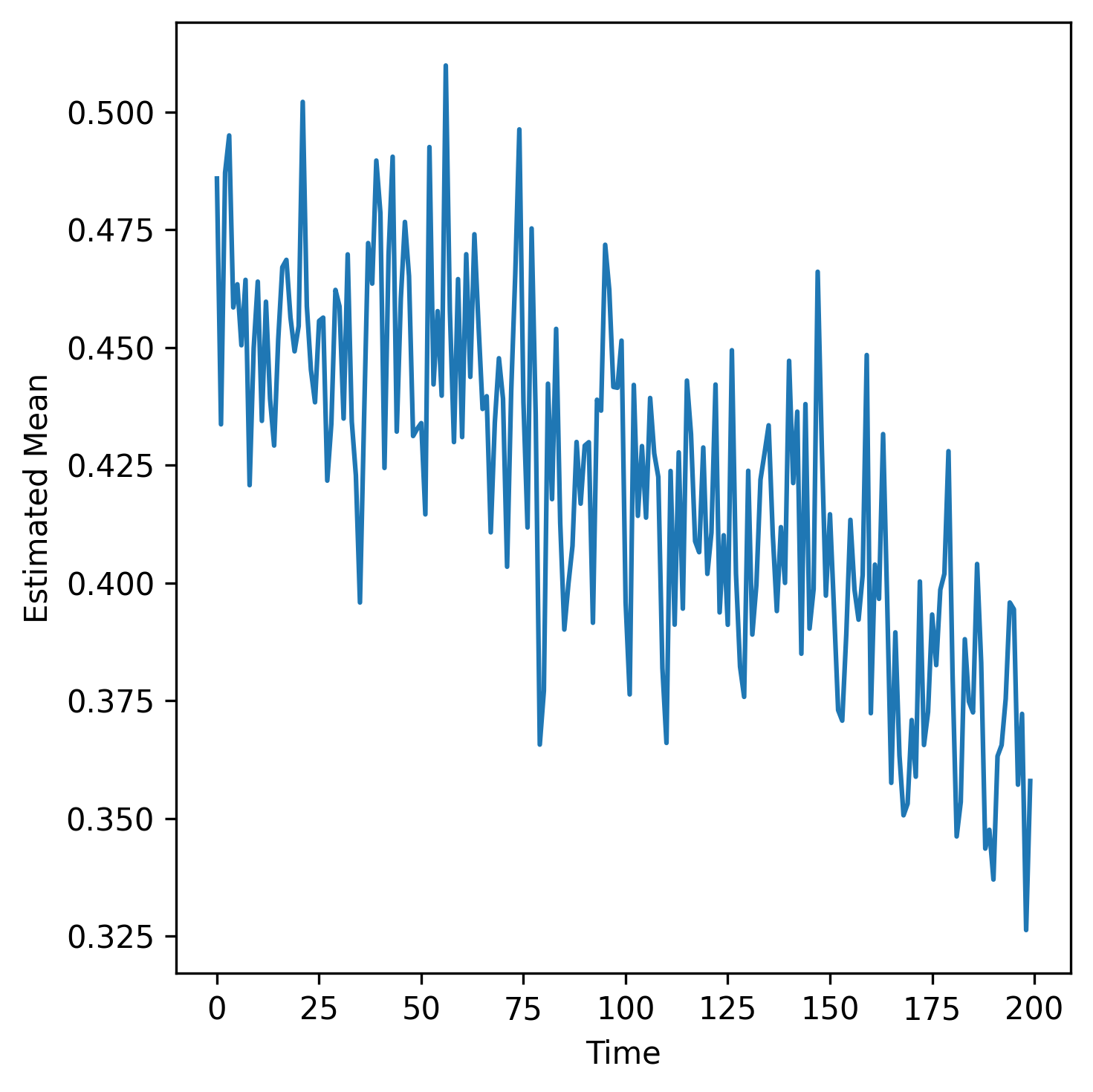}
		}
	\end{subfigure}
	\quad
	\begin{subfigure}[Dubai housing data]{
			\label{fig-housing-true-means}
			\includegraphics[scale=0.5]{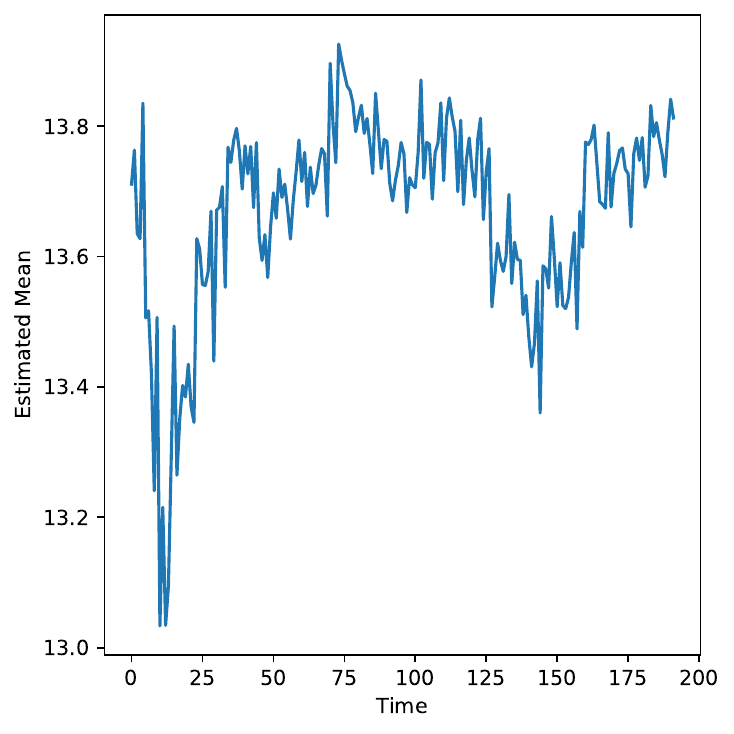}
		}
	\end{subfigure}
	\caption{Two patterns of temporal distribution shift: slow drift (left) and abrupt changes (right).}
\end{figure}

\subsection{Experiments on the Housing Data}\label{sec-housing-details}

We focus on transactions of studios and apartments with 1 to 4 bedrooms, between January 1st, 2008 and December 31st, 2023. We import variables \texttt{instance\_date} (transaction date), \texttt{area\_name\_en} (English name of the area where the apartment is located in), \texttt{rooms\_en} (number of bedrooms), \texttt{has\_parking} (whether or not the apartment has a parking spot), \texttt{procedure\_area} (area in the apartment), \texttt{actual\_worth} (final price) from the data. 

We use \texttt{instance\_date} (transaction date) to construct monthly datasets. The target for prediction is the logarithmic of \texttt{actual\_worth}. The predictors are \texttt{area\_name\_en}, \texttt{rooms\_en}, \texttt{has\_parking} and \texttt{procedure\_area}. \texttt{area\_name\_en} has 58 possible values and encoded as an integer variable. 
We remove a sample when its \texttt{actual\_worth} or \texttt{procedure\_area} is among the largest or smallest 2.5\% of the population, whichever is true. After the trimming, $91.6\%$ of the data remains.

For random forest regression, we use the function \texttt{RandomForestRegressor} in the Python library \texttt{scikit-learn}. For XGBoost regression, we use the function \texttt{XGBRegressor} in the Python library \texttt{xgboost}. In both cases, we set \texttt{random\_state = 0} and do not change any other default parameters. 
Implementation details can be found in our code at \url{https://github.com/eliselyhan/ARW}.